\documentclass[sigconf]{aamas}

\usepackage{balance} 
\usepackage[ruled,vlined,linesnumbered]{algorithm2e}
\usepackage{tikz}
\usetikzlibrary{arrows.meta, positioning, calc, shadows, shapes.geometric}
\usepackage{mathtools} 

\usepackage{helvet}
\usepackage{courier}
\usepackage{graphicx}
\usepackage{natbib}
\usepackage{caption}
\usepackage{subcaption}
\usepackage{amsthm,amsmath,amssymb}
\usepackage[ruled,vlined,linesnumbered]{algorithm2e}
\usepackage{booktabs}
\usepackage{listings}
\usepackage{pgfplots}
\pgfplotsset{compat=1.18}

\doi{CHZG9392}

\makeatletter
\gdef\@copyrightpermission{
  \begin{minipage}{0.2\columnwidth}
   \href{https://creativecommons.org/licenses/by/4.0/}{\includegraphics[width=0.90\textwidth]{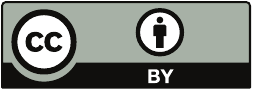}}
  \end{minipage}\hfill
  \begin{minipage}{0.8\columnwidth}
   \href{https://creativecommons.org/licenses/by/4.0/}{This work is licensed under a Creative Commons Attribution International 4.0 License.}
  \end{minipage}
  \vspace{5pt}
}
\makeatother

\setcopyright{ifaamas}
\acmConference[AAMAS '26]{Proc.\@ of the 25th International Conference
on Autonomous Agents and Multiagent Systems (AAMAS 2026)}{May 25 -- 29, 2026}
{Paphos, Cyprus}{C.~Amato, L.~Dennis, V.~Mascardi, J.~Thangarajah (eds.)}
\copyrightyear{2026}
\acmYear{2026}
\acmDOI{}
\acmPrice{}
\acmISBN{}

\title[OSL]{The Observer–Situation Lattice: 
    A Unified Formal Basis for Perspective-Aware Cognition}

\author{Saad Alqithami}
\affiliation{
  \institution{Computer Science Department, Al-Baha University}
  \city{Albaha 65779}
  \country{Saudi Arabia}}
\email{salqithami@bu.edu.sa}

\begin{abstract}
Autonomous agents operating in complex, multi-agent environments must reason about what is true from multiple perspectives. Existing approaches often struggle to integrate the reasoning of different agents, at different times, and in different contexts, typically handling these dimensions in separate, specialized modules. This fragmentation leads to a brittle and incomplete reasoning process, particularly when agents must understand the beliefs of others (Theory of Mind). We introduce the Observer–Situation Lattice (OSL), a unified mathematical structure that provides a single, coherent semantic space for perspective-aware cognition. OSL is a finite complete lattice where each element represents a unique observer-situation pair, allowing for a principled and scalable approach to belief management. We present two key algorithms that operate on this lattice: (i) Relativized Belief Propagation, an incremental update algorithm that efficiently propagates new information, and (ii) Minimal Contradiction Decomposition, a graph-based procedure that identifies and isolates contradiction components. We prove the theoretical soundness of our framework and demonstrate its practical utility through a series of benchmarks, including classic Theory of Mind tasks and a comparison with established paradigms such as assumption-based truth maintenance systems. Our results show that OSL provides a computationally efficient and expressive foundation for building robust, perspective-aware autonomous agents.
\end{abstract}


\newcommand{\BibTeX}{\rm B\kern-.05em{\sc i\kern-.025em b}\kern-.08em\TeX}

\begin{document}


\pagestyle{fancy}
\fancyhead{}


\maketitle 


\section{Introduction} \label{sec:intro}

Autonomous agents deployed in real-world, multi-agent settings must navigate a complex information landscape where truth is relative. An observation made by one agent may not be available to another; a fact true in one context may be false in another. For example, in a smart building, a maintenance robot may observe a wet floor, while a remote supervisor sees only an aggregated alert, and a security operator with access to different camera feeds perceives nothing amiss. Effective collaboration in such environments requires agents to reason not just about the world, but about the perspectives of others—a capability known as Theory of Mind \citep{gorgan2024computational}. This fundamental challenge of perspective-aware reasoning is a critical bottleneck in the development of robust, socially intelligent AI systems \citep{dorri2018multi}.

Existing agent architectures and formalisms struggle to address this challenge in a unified and scalable manner. Mainstream approaches like the Belief-Desire-Intention (BDI) model \citep{rao1995bdi,rao1991modeling} provide a strong foundation for single-agent deliberation but typically assume a single, globally consistent belief state, making it difficult to manage the multiple, often conflicting, viewpoints inherent in multi-agent systems. While various BDI programming languages and platforms exist \citep{dastani2008apl, nunes2011bdi4jade, sardina2011bdi}, they often treat perspective-taking as an add-on rather than a core architectural principle. Similarly, foundational work in epistemic logic \citep{hintikka1962knowledge, Fagin2003, vanDitmarsch2007, baltag2016dynamic, van2015introduction} provides a formal language for reasoning about knowledge and belief, but its computational complexity often renders it impractical for dynamic, real-time applications \citep{bolander2020del, aucher2013undecidability}. Other approaches, such as multi-context systems \citep{brewka2011managed} and truth maintenance systems \citep{doyle1979truth, de1986assumption, dechter1996structure}, offer powerful tools for managing different belief sets but lack a unifying semantic structure that can seamlessly integrate the notions of observer, context, and time.

This fragmentation of approaches leads to a critical research gap: the absence of a computationally efficient, formally grounded framework that treats perspective as a first-class citizen in an agent's cognitive architecture. To address this gap, we introduce the Observer-Situation Lattice (OSL), a novel mathematical structure that provides a single, unified semantic space for perspective-aware reasoning. The OSL is a finite complete lattice where each element represents a unique observer-situation pair. This allows an agent to represent and reason about facts of the form ``observer $o$ in situation $\sigma$ believes $\varphi$'' within a single, coherent algebraic structure. The lattice's partial order naturally captures knowledge containment relationships between observers and contextual refinement between situations, providing a principled foundation for belief propagation and contradiction management.

This paper makes the following key contributions:
\begin{enumerate}
    \item We formalize the OSL as a finite product lattice and prove its key mathematical properties, demonstrating that it provides a sound and complete foundation for perspective-aware reasoning (Section~\ref{sec:osl-framework}).
    \item We introduce two algorithms that operate on the OSL: Relativized Belief Propagation, an efficient incremental update algorithm, and Minimal Contradiction Decomposition, a graph-based procedure that identifies and isolates contradiction components (Sections~\ref{subsec:rbp} and~\ref{subsec:mcc}).
    \item We demonstrate how the OSL can be integrated into a practical agent architecture, providing a concrete pathway for implementing perspective-aware reasoning in BDI-style agents (Section~\ref{subsec:architecture}).
    \item We validate our approach through a series of experiments, including classic Theory of Mind benchmarks and a comparative analysis against established paradigms, showing that OSL offers significant expressive and computational advantages (Section~\ref{sec:experiments}).
\end{enumerate}

By unifying the dimensions of observer, situation, and belief within a single lattice structure, OSL provides a powerful and elegant solution to the long-standing problem of perspective-aware reasoning in multi-agent systems. It offers a clear path toward building more robust, socially intelligent, and ultimately more effective autonomous agents.

\section{Background and Related Work}
\label{sec:background}

The challenge of perspective-aware reasoning lies at the intersection of several established fields in artificial intelligence. Our work builds upon and extends concepts from cognitive architectures, knowledge representation, and multi-agent systems. In this section, we situate OSL within this broader landscape, highlighting the key limitations of existing approaches that motivate our unified framework.

\subsection{Cognitive Architectures and BDI Systems}

Cognitive architectures aim to provide a comprehensive blueprint for intelligent behavior. Classic symbolic architectures like Soar \citep{laird2022analysis} and ACT-R \citep{anderson2004integrated} have demonstrated impressive reasoning capabilities but have traditionally focused on a single agent's cognitive processes, with limited built-in support for multi-agent perspective-taking. More recently, there has been a push to integrate neural and symbolic approaches, leading to the development of neuro-symbolic architectures \citep{marra2024from, wan2024towards} that promise greater flexibility and learning capabilities. However, the challenge of managing multiple, potentially conflicting belief sets in a principled manner remains.

The Belief-Desire-Intention (BDI) model \citep{rao1991modeling, rao1995bdi} has been particularly influential in the multi-agent systems community. BDI agents are characterized by their explicit representation of mental states, making them a natural fit for tasks requiring deliberation and goal-directed behavior. Practical BDI programming languages like AgentSpeak, implemented in platforms such as Jason \citep{bordini2007programming}, have enabled the development of complex multi-agent systems. However, the standard BDI model does not inherently account for the different perspectives of other agents. While various extensions have been proposed to incorporate Theory of Mind capabilities, they often do so in an ad-hoc manner, bolting on separate modules for reasoning about the beliefs of others. This can lead to a fragmented and computationally inefficient architecture, a key limitation that OSL is designed to address.

\subsection{Knowledge Representation and Reasoning}

At its core, perspective-aware reasoning is a problem of knowledge representation. How can we represent what an agent believes, in a given context, at a particular time? Epistemic logic, originating with the work of Hintikka \citep{hintikka1962knowledge}, provides a formal framework for reasoning about knowledge and belief. Dynamic Epistemic Logic (DEL) extends this framework to model how knowledge changes as a result of actions and observations \citep{van2007dynamic}. While DEL offers a powerful theoretical tool, its application in practical systems has been limited by its high computational complexity, particularly in multi-agent scenarios \citep{bolander2020del}. Epistemic planning, which aims to find a sequence of actions to achieve a knowledge-based goal, faces similar scalability challenges \citep{aucher2013undecidability}.

Truth Maintenance Systems (TMS) offer a more computational approach to managing beliefs. Justification-based TMS (JTMS) \citep{doyle1979truth} and Assumption-based TMS (ATMS) \citep{de1986assumption} provide mechanisms for tracking dependencies between beliefs and efficiently revising them when contradictions are found. While ATMS can manage multiple belief sets (contexts) simultaneously, it does not provide a semantic structure for relating these contexts to different observers or situations. Distributed TMS (DTMS) approaches have been developed for multi-agent settings \citep{bridgeland1990distributed, huhns1991multiagent}, but they typically focus on achieving a single, globally consistent state rather than managing a diversity of perspectives.

Our work is also related to Formal Concept Analysis (FCA), a mathematical framework for deriving conceptual hierarchies from data \citep{ganter1999formal}. FCA uses lattice theory to organize concepts, and its application to knowledge representation has been explored in various contexts \citep{kuznetsov2013knowledge}. OSL draws inspiration from FCA in its use of a lattice structure, but it applies this structure in a novel way to model the relationships between observers and situations, providing a dynamic framework for belief management rather than a static analysis of concepts.

\subsection{Theory of Mind and Explainable AI}

Reasoning about the mental states of others, or Theory of Mind (ToM), is a cornerstone of social intelligence. The ability to attribute false beliefs to others, as famously demonstrated in the Sally-Anne test, is a key benchmark for ToM. Computational models of ToM have been developed in various fields, from cognitive science to robotics \citep{gorgan2024computational, labash2020perspective}, but they often rely on explicit, nested representations of beliefs (e.g., ``I believe that you believe that... "), which can become computationally intractable. OSL offers a more elegant solution by representing all beliefs within a single, flat lattice structure, where reasoning about higher-order beliefs can be performed through simple lattice operations.

Finally, the ability to explain one's reasoning is crucial for effective human-agent collaboration. The field of Explainable AI (XAI) has made significant strides in developing techniques for generating explanations \citep{dosilovic2018explainable, tjoa2021survey}. However, most XAI approaches are not ``audience-aware'' and generate a single explanation without considering the knowledge or perspective of the person receiving it. OSL's explicit representation of observer perspectives provides a natural foundation for generating personalized explanations that are tailored to the specific knowledge and context of the observer \citep{alshomary2021toward, vasileiou2025monolithic}.


Throughout the paper we use the following conventions. Let $\mathcal{O}$ be the set of observers equipped with a partial order $\preceq_{\mathcal{O}}$, and let $\Sigma$ be the set of situations with partial order $\preceq_{\Sigma}$. We write $\mathcal{E} = \mathcal{O} \times \Sigma$ for the OSL carrier, ordered by the product order $\preceq$ induced by $\preceq_{\mathcal{O}}$ and $\preceq_{\Sigma}$. A belief base is denoted by $B$, and we let $n = |\mathcal{E}|$ denote the size of the lattice.

\section{The Observer-Situation Lattice Framework}
\label{sec:osl-framework}

The OSL framework is built upon the mathematical theory of lattices, providing a robust and principled foundation for perspective-aware reasoning. In this section, we detail the formal construction of the OSL, define its belief semantics, and present the core algorithms for belief propagation and contradiction management.

\subsection{Formal Mathematical Foundations}

The OSL framework is constructed as a product of two fundamental partial orders that capture the essential relationships in perspective-aware reasoning systems. We begin with the mathematical foundations that ensure the framework's theoretical soundness.

\begin{definition}[Observer Knowledge Containment]
\label{def:obs_order}
Let $\mathcal{O}$ be a finite set of observers and $\preceq_{\mathcal{O}} \subseteq \mathcal{O} \times \mathcal{O}$ be a partial order expressing informational containment. We write $o_{1} \preceq_{\mathcal{O}} o_{2}$ to denote that observer $o_{2}$ has at least the knowledge available to observer $o_{1}$. This order may be derived from trust hierarchies, sensor fusion graphs, capability inclusion relationships, or explicit authority structures.
\end{definition}

The observer order captures the fundamental asymmetries in multi-agent systems where different agents possess varying levels of information access, computational capabilities, or epistemic authority. This ordering is not merely theoretical but reflects practical considerations such as sensor quality, communication bandwidth, processing power, and domain expertise.

\begin{definition}[Situation Refinement]
\label{def:sit_order}
Let $\Sigma$ be a finite set of situations and $\preceq_{\Sigma} \subseteq \Sigma \times \Sigma$ be a partial order capturing refinement of contextual situations. We write $\sigma_{1} \preceq_{\Sigma} \sigma_{2}$ when every fact that is true in situation $\sigma_{1}$ is also true in situation $\sigma_{2}$, meaning that $\sigma_{2}$ represents a more specific or constrained context than $\sigma_{1}$ (i.e., $\sigma_{2}$ refines $\sigma_{1}$).
\end{definition}

The situation refinement order embodies the principle that more specific contexts inherit all properties of their generalizations while potentially adding additional constraints or details. This captures the natural hierarchy of contextual information, from broad environmental conditions to specific task-oriented scenarios.

Following standard lattice theory \citep{ganter1999formal}, we assume that both $\langle\mathcal{O}, \preceq_{\mathcal{O}}\rangle$ and $\langle\Sigma, \preceq_{\Sigma}\rangle$ are finite complete lattices. This ensures that for any subset of observers or situations, there exists a unique least upper bound (join) and a unique greatest lower bound (meet). This completeness property is not merely a theoretical convenience; it is the essential mathematical property that guarantees the algorithmic tractability of our belief propagation and contradiction resolution procedures.

\begin{definition}[Observer-Situation Order]
\label{def:osl_order}
For elements $e_{1}, e_{2} \in \mathcal{E}$ where $e_{1} = (o_{1}, \sigma_{1})$ and $e_{2} = (o_{2}, \sigma_{2})$, define the product order as:
\[
e_{1} \preceq e_{2} \quad \Longleftrightarrow \quad (o_{1} \preceq_{\mathcal{O}} o_{2}) \land (\sigma_{1} \preceq_{\Sigma} \sigma_{2})
\]
The structure $\langle\mathcal{E}, \preceq\rangle$ is called the Observer-Situation Lattice (OSL) carrier.
\end{definition}

We say that two lattice elements $e_1,e_2 \in \mathcal{E}$ are comparable, written $e_1 \bowtie e_2$, if $e_1 \preceq e_2$ or $e_2 \preceq e_1$. We reserve $\preceq$ for the lattice order and $\bowtie$ for this comparability relation.

This product construction ensures that the OSL inherits the mathematical properties of its component lattices while providing a natural framework for reasoning about the interaction between observer capabilities and situational contexts.

\begin{lemma}[Product completeness]
\label{lem:product-complete}
Let $\langle O,\preceq_O\rangle$ and $\langle \Sigma,\preceq_\Sigma\rangle$ be finite complete lattices. 
Define $E = O \times \Sigma$ with the component-wise order
\[
(o_1,\sigma_1) \preceq (o_2,\sigma_2) 
\quad\Longleftrightarrow\quad
o_1 \preceq_O o_2 \ \text{ and }\ \sigma_1 \preceq_\Sigma \sigma_2.
\]
Then $\langle E,\preceq\rangle$ is a finite complete lattice. Moreover, for any $S \subseteq E$,
\[
\bigvee S = \left(\bigvee_{(o,\sigma)\in S} o,\ \bigvee_{(o,\sigma)\in S} \sigma\right),
\qquad
\bigwedge S = \left(\bigwedge_{(o,\sigma)\in S} o,\ \bigwedge_{(o,\sigma)\in S} \sigma\right).
\]
\end{lemma}

\begin{proof}[Proof Sketch]
Completeness of $O$ and $\Sigma$ guarantees that all joins and meets in the right-hand sides above exist and are unique. The construction is the standard product of complete lattices, which is known to be complete under the component-wise order; we spell out the join and meet to make later algorithmic use explicit.
\end{proof}

\begin{theorem}[OSL completeness]
\label{thm:osl-complete}
Assume $\langle O,\preceq_O\rangle$ and $\langle \Sigma,\preceq_\Sigma\rangle$ are finite complete lattices and let $E=O\times\Sigma$ with the order from Lemma~\ref{lem:product-complete}. Then $\langle E,\preceq\rangle$ is a finite complete lattice. In particular, every subset $S\subseteq E$ has a join and a meet, computable in $O(|O|\,|\Sigma|)$ time by scanning all elements.
\end{theorem}

\begin{proof}[Proof Sketch]
The proof follows directly from the standard construction of product lattices \citep{wille1982restructuring}. Since $\mathcal{O}$ and $\Sigma$ are finite complete lattices by assumption, their product $\langle\mathcal{E}, \preceq\rangle$ is also a finite complete lattice. The computational complexity of finding the join or meet of a set of elements is determined by the need to iterate through the component sets, which in the worst case requires scanning all elements of $\mathcal{O}$ and $\Sigma$.

The existence and uniqueness of arbitrary suprema and infima ensures that the OSL provides a mathematically sound foundation for belief aggregation and conflict resolution, as these operations correspond directly to lattice-theoretic suprema and infima.
\end{proof}

\subsection{Belief Semantics and Truth Propagation}

The lattice structure of the OSL provides a natural semantics for representing and propagating beliefs in a way that respects both observer capabilities and situational contexts.

\begin{definition}[Belief Record]
\label{def:belief_record}
A belief record is a triple $\langle \varphi, e, w \rangle$ where $\varphi$ is a propositional formula from a fixed language $\mathcal{L}$, $e \in \mathcal{E}$ is a lattice element representing the observer-situation context, and $w \in [0,1]$ is a credibility weight representing the strength of the belief. A belief base $B$ is a finite set of belief records.
\end{definition}

While we use propositional logic for $\mathcal{L}$ to maintain computational tractability, the framework can be extended to more expressive logics, such as first-order or modal logics, with corresponding increases in computational complexity. This trade-off between expressiveness and tractability is a common theme in knowledge representation \citep{Reiter2001}.

\begin{definition}[Upward Closure Semantics]
\label{def:truth_eval}
Given a belief base $B$ and lattice element $e \in \mathcal{E}$, the credibility of formula $\varphi$ at $e$ is defined as:
\[
\text{cred}(\varphi, e, B) = \max\{w : \langle \varphi, e', w \rangle \in B \text{ and } e' \preceq e\}
\]
If no such record exists, $\text{cred}(\varphi, e, B) = 0$. The set of all belief records that contribute to the credibility of $\varphi$ at $e$ is called the \emph{support set} of $\varphi$ at $e$.
\end{definition}

This semantics captures monotone inheritance along the lattice order: if $e' \preceq e$, then evidence asserted at $e'$ is available at $e$. Intuitively, contexts corresponding to more informed observers and more refined situations inherit beliefs from less informed and coarser contexts. The maximum operation ensures that the strongest available evidence is used when multiple sources provide information about the same proposition.

\begin{lemma}[Monotonicity of Credibility]
\label{lem:cred_monotone}
For any formula $\varphi$, belief base $B$, and lattice elements $e_1, e_2 \in \mathcal{E}$ with $e_1 \preceq e_2$, we have $\text{cred}(\varphi, e_1, B) \leq \text{cred}(\varphi, e_2, B)$.
\end{lemma}

\begin{proof}[Proof Sketch]
By definition, $\text{cred}(\varphi, e_1, B)$ considers only belief records $\langle \varphi, e', w \rangle$ where $e' \preceq e_1$. Since $e_1 \preceq e_2$, by transitivity of $\preceq$, any such $e'$ also satisfies $e' \preceq e_2$. Therefore, the set of records contributing to $\text{cred}(\varphi, e_1, B)$ is a subset of those contributing to $\text{cred}(\varphi, e_2, B)$, implying the desired inequality.
\end{proof}

This monotonicity property ensures that the credibility function respects the lattice structure and provides a foundation for efficient belief propagation algorithms.

\subsection{Algorithmic Framework}

The mathematical foundations enable the development of efficient algorithms for belief propagation and contradiction detection that exploit the lattice structure for computational advantage.

\subsubsection{Relativized Belief Propagation (RBP)} \label{subsec:rbp}

The RBP algorithm efficiently updates credibility values when new belief records are inserted into the lattice, leveraging the upward closure property to minimize computational overhead.

\begin{algorithm}[ht]
\caption{Relativized Belief Propagation (RBP)}
\label{alg:rbp}
\KwIn{Lattice element $e = (o, \sigma)$, belief record $\langle \varphi, e, w \rangle$, belief base $B$}
\KwOut{Updated credibility values for affected elements}

Insert $\langle \varphi, e, w \rangle$ into belief base $B$\;
Compute upward closure $\uparrow e = \{e' \in \mathcal{E} : e \preceq e'\}$\;
Initialize update queue $Q \leftarrow \uparrow e$\;
Initialize affected set $A \leftarrow \emptyset$\;
\While{$Q \neq \emptyset$}{
    $e' \leftarrow \text{dequeue}(Q)$\;
    $w_{\text{old}} \leftarrow \text{cred}(\varphi, e', B \setminus \{\langle \varphi, e, w \rangle\})$\;
    $w_{\text{new}} \leftarrow \text{cred}(\varphi, e', B)$\;
    \If{$w_{\text{new}} > w_{\text{old}}$}{
        Update cached credibility value for $(\varphi, e')$\;
        $A \leftarrow A \cup \{e'\}$\;
        Notify dependent reasoning modules of update\;
    }
}
\Return{affected set $A$}\;
\end{algorithm}

This incremental approach is significantly more efficient than recomputing all credibility values from scratch, particularly in large lattices where new information may only affect a small subset of the belief space.

\begin{theorem}[RBP Correctness]
\label{thm:rbp_correct}
Algorithm~\ref{alg:rbp} correctly computes the updated credibility values for all lattice elements after inserting a new belief record, maintaining consistency with Definition~\ref{def:truth_eval}.
\end{theorem}

\begin{proof}[Proof Sketch]
The algorithm computes the upward closure $\uparrow e$ of the insertion point, which by Definition~\ref{def:truth_eval} contains exactly those elements whose credibility for $\varphi$ might be affected by the new record. For each element $e' \in \uparrow e$, the algorithm recomputes $\text{cred}(\varphi, e', B)$ by considering all records $\langle \varphi, e'', w'' \rangle$ where $e'' \preceq e'$, which now includes the newly inserted record if $e \preceq e'$. The maximum operation in Definition~\ref{def:truth_eval} ensures that the strongest evidence is selected, maintaining correctness.
\end{proof}

\begin{theorem}[RBP complexity]
\label{thm:rbp-complexity}
For an insertion at lattice element $e\in \mathcal{E}$, RBP visits exactly the elements in the upward closure
$\uparrow e = \{e' \in \mathcal{E} : e \preceq e'\}$.
Its running time is $O(|\uparrow e|)$ and therefore $O(|\mathcal{E}|)$ in the worst case.
\end{theorem}

\begin{proof}[Sketch]
By Definition~\ref{def:truth_eval}, RBP restricts updates to $\uparrow e$ and performs $O(1)$ work per visited node (credibility recomputation and cache update). The size of $\uparrow e$ is at most $|\mathcal{E}|$, giving the stated worst-case bound.
\end{proof}

Section~\ref{sec:experiments} shows that on balanced lattices, the empirical cost grows strictly sub-linearly in $|\mathcal{E}|$, with log--log slopes between $0.34$ and $0.42$ on lattices up to $10^5$ elements.

\subsubsection{Enhanced RBP with Convergence Analysis}

The basic RBP algorithm can be enhanced with iterative refinement to handle complex belief dependencies and ensure convergence in the presence of cycles.

\begin{algorithm}[ht]
\caption{Enhanced RBP with Convergence}
\label{alg:rbp_enhanced}
\KwIn{Belief base $B$, convergence threshold $\epsilon > 0$, maximum iterations $K$}
\KwOut{Converged credibility values}

Initialize credibility matrix $C[e, \varphi] \leftarrow 0$ for all $e \in \mathcal{E}, \varphi \in \mathcal{L}$\;
$k \leftarrow 0$\;
\Repeat{$\Delta < \epsilon$ or $k \geq K$}{
    $C_{\text{old}} \leftarrow C$\;
    \For{each $(e, \varphi)$ pair}{
        $C[e, \varphi] \leftarrow \max\{w : \langle \varphi, e', w \rangle \in B, e' \preceq e\}$\;
    }
    $\Delta \leftarrow \max_{e,\varphi} |C[e, \varphi] - C_{\text{old}}[e, \varphi]|$\;
    $k \leftarrow k + 1$\;
}
\Return{credibility matrix $C$}\;
\end{algorithm}

\begin{theorem}[Enhanced RBP Convergence]
\label{thm:rbp_convergence}
Algorithm~\ref{alg:rbp_enhanced} converges to a fixed point that equals the credibility matrix induced by Definition~\ref{def:truth_eval}. For the update rule in line~6, convergence occurs in at most two iterations (one update sweep and one fixed-point check).
\end{theorem}

\begin{proof}[Proof Sketch]
In each iteration, line~6 assigns
$C[e,\varphi] \leftarrow \max\{w : \langle \varphi,e',w\rangle\in B \text{ and } e'\preceq e\}$,
which is exactly $\mathrm{cred}(\varphi,e,B)$ by Definition~\ref{def:truth_eval} and does not depend on the previous value of $C$. Therefore, after one complete sweep over all $(e,\varphi)$ pairs, $C$ equals the target credibility matrix. A second sweep leaves $C$ unchanged, so $\Delta=0$ and the loop terminates.

More generally, if Algorithm~\ref{alg:rbp_enhanced} is extended with additional monotone update dependencies (e.g., rule-based derived beliefs), standard fixed-point iteration on a finite complete lattice converges in finitely many steps to the least fixed point (Knaster--Tarski).
\end{proof}

\subsubsection{Minimal Contradiction Decomposition (MCC)} \label{subsec:mcc}

Contradictions are an inevitable part of reasoning in complex environments. The MCC algorithm provides a principled way to identify and isolate contradictions by leveraging the structure of the OSL.

\begin{algorithm}[ht]
\caption{Minimal Contradiction Decomposition (MCC)}
\label{alg:mcc}
\KwIn{Belief base $B$, lattice $\mathcal{E}$}
\KwOut{Set of contradiction components $\mathcal{C}$}

Initialize contradiction graph $G = (V, E)$ where $V = B$ and $E = \emptyset$\;
\For{each pair of belief records $b_1, b_2 \in B$}{
    Extract $\langle \varphi_1, e_1, w_1 \rangle = b_1$ and $\langle \varphi_2, e_2, w_2 \rangle = b_2$\;
    \If{$e_1$ and $e_2$ are comparable in the lattice}{
        \If{$\textsc{Contradict}(\varphi_1,\varphi_2)$}{
            Add edge $(b_1, b_2)$ to $E$\;
        }
    }
}
Compute connected components $\mathcal{C} = \{C_1, C_2, \ldots, C_k\}$ of $G$\;
\For{each component $C_i \in \mathcal{C}$}{
    \If{$|C_i| = 1$}{
        Remove $C_i$ from $\mathcal{C}$ \tcp{Single beliefs cannot be contradictory}
    }
}
\Return{contradiction components $\mathcal{C}$}\;
\end{algorithm}

The algorithm first constructs a contradiction graph where vertices represent belief records and an edge exists between two records if (i) their observer--situation contexts are comparable in the lattice and (ii) their formulas are contradictory according to a contradiction predicate $\textsc{Contradict}(\cdot,\cdot)$. In our benchmarks, formulas are literals and $\textsc{Contradict}(\varphi_1,\varphi_2)$ reduces to a constant-time syntactic check ($\varphi_1 \equiv \neg\varphi_2$). More generally, for arbitrary propositional formulas one can instantiate $\textsc{Contradict}(\varphi_1,\varphi_2)$ via satisfiability testing (e.g., $\varphi_1 \wedge \varphi_2$ unsatisfiable), at the cost of NP-complete worst-case complexity.

The connected components of the contradiction graph correspond to \emph{contradiction components}: clusters of belief records that are linked by (possibly transitive) chains of contradictions. This isolates independent inconsistency regions, since no contradiction edge crosses between distinct components.

\begin{theorem}[MCC Correctness and Complexity]
\label{thm:mcc_complexity}
Algorithm~\ref{alg:mcc} correctly computes the contradiction components induced by contradiction edges between comparable belief records.
Let $T_{\textsc{Contradict}}$ denote the worst-case time to evaluate $\textsc{Contradict}(\varphi_1,\varphi_2)$. The worst-case running time is $O(|B|^2\,T_{\textsc{Contradict}} + |B|\,\alpha(|B|))$, where $\alpha$ is the inverse Ackermann function.
\end{theorem}

\begin{proof}[Proof Sketch]
Correctness: the algorithm adds an edge between $b_1$ and $b_2$ exactly when their contexts are comparable and $\textsc{Contradict}(\varphi_1,\varphi_2)$ holds. Therefore, any pair of belief records that can directly contradict (under the chosen contradiction predicate) appears as an edge in $G$. Computing connected components then yields the equivalence classes under reachability in $G$, i.e., the contradiction components.

Complexity: the algorithm examines all pairs of belief records, requiring $O(|B|^2)$ context-comparability checks and at most $O(|B|^2)$ evaluations of $\textsc{Contradict}$, giving $O(|B|^2\,T_{\textsc{Contradict}})$. Connected components can be computed with union--find in $O(|B|\,\alpha(|B|))$ time.
\end{proof}

In practice, the lattice structure strongly restricts the number of comparable pairs $e_1 \bowtie e_2$, so the number of edges in the contradiction graph is often much smaller than $|B|^2$. Combined with the constant-time literal contradiction test used in our benchmarks, the empirical behavior in Section~\ref{sec:experiments} is close to $O(|B|\log|B|)$ on our tasks.

\subsubsection{Integrated Belief Management}

The RBP and MCC algorithms can be integrated into a unified belief management system that maintains consistency while supporting efficient updates.

\begin{algorithm}[ht]
\caption{Integrated OSL Belief Management}
\label{alg:integrated}
\KwIn{New belief record $\langle \varphi, e, w \rangle$, current belief base $B$}
\KwOut{Updated consistent belief base $B'$}

$B' \leftarrow B \cup \{\langle \varphi, e, w \rangle\}$\;
Run RBP to propagate the new belief (Algorithm~\ref{alg:rbp})\;
$\mathcal{C} \leftarrow$ MCC$(B', \mathcal{E})$ \tcp{Detect contradictions}
\If{$\mathcal{C} \neq \emptyset$}{
    \For{each contradiction component $C \in \mathcal{C}$}{
        Resolve contradiction using credibility-based selection\;
        Remove lower-credibility beliefs from $B'$\;
    }
    Re-run RBP to propagate consistency updates\;
}
\Return{updated belief base $B'$}\;
\end{algorithm}

This integrated approach ensures that the belief base remains both complete (through RBP propagation) and consistent (through MCC contradiction resolution) while minimizing computational overhead through incremental updates.

\subsection{Architectural Integration} \label{subsec:architecture}

A key advantage of the OSL framework is its ability to serve as the central backbone of a unified agent architecture. By providing a single, shared structure for belief management, OSL eliminates the need for separate, ad-hoc modules for context management, Theory of Mind, or explanation generation. This leads to a more elegant, robust, and computationally efficient design.

\subsubsection{Unified Perspective-Aware Architecture}

Figure~\ref{fig:architecture} illustrates how OSL can be integrated into a BDI-style agent architecture. In this model, the OSL belief base serves as the central repository of all propositional knowledge, accessible to all other cognitive modules. This design is inspired by Global Workspace Theory \citep{vanrullen2021deep, Franklin2014}, where a central workspace broadcasts information to a collection of specialized modules.

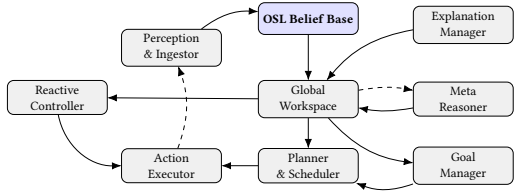
\begin{figure}[ht]
\centering
\resizebox{0.8\linewidth}{!}{
\begin{tikzpicture}[
    node distance=8mm and 8mm,
    every node/.style={font=\small},
    module/.style={draw,rounded corners,align=center,minimum width=22mm,minimum height=7mm,fill=gray!12},
    arrow/.style={-{Latex[length=2mm]},line width=0.4pt}
]
\node[module] (percept) {Perception\\\& Ingestor};
\node[module,below left= 3mm and 3mm of percept] (react) {Reactive\\Controller};
\node[module,below right=18mm and 8mm of percept] (planner) {Planner\\\& Scheduler};
\node[module,left=of planner] (executor) {Action\\Executor};
\node[module,right=12mm of planner] (goal) {Goal\\Manager};
\node[module,below right=3mm and 8mm of percept] (workspace) {Global\\Workspace};
\node[module,above right=8mm and 12mm of workspace] (expl) {Explanation\\Manager};
\node[module,right=12mm of workspace] (meta) {Meta\\Reasoner};
\node[module,fill=blue!12,above=of workspace,yshift=2mm,minimum height=7mm] (osl) {\textbf{OSL Belief Base}};

\draw[arrow] (percept) edge[bend left=20] (osl);
\draw[arrow] (react) edge[bend left=-40] (executor.west);
\draw[arrow] (planner) edge (executor);
\draw[arrow] (goal) edge[bend left=20] (planner);
\draw[arrow] (workspace) edge[bend left=-20] (goal);
\draw[arrow] (workspace) edge (planner);
\draw[arrow] (workspace) edge (react);
\draw[arrow] (osl) edge (workspace);
\draw[arrow,dashed] (executor) edge[bend left=-20] (percept);
\draw[arrow] (meta) edge[bend left=10] (workspace);
\draw[arrow,dashed] (workspace) edge[bend left=10] (meta);
\draw[arrow] (expl) edge[bend left=-20] (workspace);
\end{tikzpicture}}
\caption{OSL-based agent architecture. Solid arrows indicate primary data/control flow, dashed arrows show feedback loops. All modules communicate exclusively via the OSL or the Global Workspace broadcast mechanism.}
\label{fig:architecture}
\end{figure}

The architecture provides several key advantages over traditional approaches:

\textbf{Unified Data Model:} All reasoning modules operate on the same lattice-structured belief base, eliminating the need for data translation between modules and reducing the potential for inconsistencies.

\textbf{Natural Context Awareness:} The lattice structure automatically provides context-sensitive access to beliefs, allowing each module to access information appropriate to its observer-situation context without explicit context management code.

\textbf{Efficient Belief Propagation:} Updates from any module are automatically propagated to all relevant contexts through the RBP algorithm, ensuring system-wide consistency without manual synchronization.

\textbf{Integrated Contradiction Handling:} The MCC algorithm provides system-wide contradiction detection and resolution, preventing inconsistencies from propagating across module boundaries.

The agent's cognitive cycle proceeds through a series of phases, each of which interacts with the OSL. New perceptual information is first added to the OSL, and the RBP algorithm propagates its consequences throughout the lattice. The planning module can then query the OSL to obtain context-sensitive information for deliberation, and the action execution module can update the OSL with the results of its actions. This tight integration of the OSL into the agent's cognitive loop ensures that all reasoning is grounded in a consistent and up-to-date representation of the world from all relevant perspectives.

\begin{figure}[ht]
\centering
\resizebox{.5\linewidth}{!}{
\begin{tikzpicture}[
  node distance = 4mm and 4mm,
  every node/.style = {font=\scriptsize,align=center},
  phase/.style = {draw,rounded corners=2pt,minimum height=7mm,minimum width=15mm,fill=gray!10},
  arrow/.style = {-{Latex[length=2mm]},line width=.4pt},
  meta/.style = {arrow,dashed}
]
\node[phase] (obs) {{\bf 1}.~Observe\\\textit{(perception)}};
\node[phase,below=of obs] (rbp) {{\bf 2}.~RBP\\propagate};
\node[phase,right=of rbp] (ws) {{\bf 3}.~Attend\\workspace};
\node[phase,right=of ws] (goal) {{\bf 4}.~Goal\\update};
\node[phase,above right=4mm and -6mm of goal] (plan) {{\bf 5a}.~Plan\\(deliberate)};
\node[phase,above=18mm of goal] (act) {{\bf 6}.~Actuate};
\node[phase,left=of plan] (react) {{\bf 5b}.~React\\(fast)};
\node[phase,below left=4mm and -6mm of goal] (reflect) {{\bf 7}.~Reflect\\\& learn};

\draw[arrow] (obs) -- (rbp);
\draw[arrow] (rbp) -- (ws);
\draw[arrow] (ws) -- (goal);
\draw[arrow] (goal) edge[bend left=-20] (plan);
\draw[arrow] (goal) edge[bend right=-20] (react);
\draw[arrow] (plan) edge[bend left=-20] (act);
\draw[arrow] (react) edge[bend left=20] (act);
\draw[arrow] (act) edge[bend left=-35] node[pos=.18,below,xshift=0mm]{\tiny feedback} (obs);
\draw[meta] (ws) edge[bend right=50] (reflect);
\draw[meta] (reflect) edge[bend right=50] (goal);
\end{tikzpicture} }
\caption{OSL cognitive cycle with nominal 100ms period. Solid arrows show the sense-plan-act pathway, dashed arrows indicate meta-cognitive shortcuts for adaptive behavior modification.}
\label{fig:cogcycle}
\end{figure}
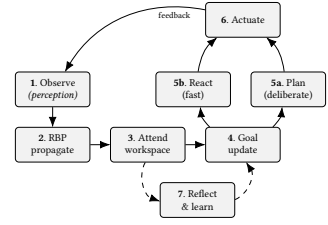

Each phase of the cognitive cycle leverages the OSL structure:

\textbf{Phase 1 - Observe:} Perceptual input is tagged with appropriate observer-situation contexts and inserted into the lattice using Algorithm~\ref{alg:rbp}.

\textbf{Phase 2 - RBP Propagate:} Belief propagation ensures that new perceptual information is available to all relevant reasoning contexts.

\textbf{Phase 3 - Attend:} The global workspace selects high-credibility beliefs from appropriate lattice elements for conscious processing.

\textbf{Phase 4 - Goal Update:} Goal management operates on goal-relevant lattice elements, updating objectives based on current context.

\textbf{Phase 5a/5b - Plan/React:} Both deliberative planning and reactive control access context-appropriate beliefs through lattice queries.

\textbf{Phase 6 - Actuate:} Action execution updates the lattice with action outcomes and environmental feedback.

\textbf{Phase 7 - Reflect:} Meta-cognitive processes analyze belief patterns across lattice elements to identify learning opportunities.

\subsubsection{Implementation Complexity Analysis}

To validate the architectural claims, we analyze the computational complexity of the integrated OSL-based architecture compared to traditional approaches.

\begin{theorem}[Architectural Complexity Bounds]
\label{thm:arch_complexity}
An OSL-based agent architecture with $m$ reasoning modules, $n$ lattice elements, and $b$ belief records has the following complexity characteristics:

\textbf{Belief Update:} $O(|\uparrow e| + \log b)$ per update, i.e., $O(n + \log b)$ in the worst case

\textbf{Context Query:} $O(\log n + \log b)$ per query

\textbf{Contradiction Detection:} $O(b^2\,T_{\textsc{Contradict}} + b\,\alpha(b))$ in the worst case, typically close to $O(b \log b)$ on localized contradictions with constant-time literal checks

\textbf{Memory Usage:} $O(n + b + m \log n)$ total
\end{theorem}

\begin{proof}[Proof Sketch]
\textbf{Belief Update:}
By Theorem~\ref{thm:rbp-complexity}, RBP restricts updates to the upward closure $\uparrow e$ of the insertion point and performs $O(1)$ work per visited node, giving a cost $O(|\uparrow e|)$ per update and $O(n)$ in the worst case since $|\uparrow e| \leq n$. Maintaining an indexed belief store adds an $O(\log b)$ insertion cost, yielding $O(|\uparrow e| + \log b)$ and therefore $O(n + \log b)$ in the worst case. Empirically, $|\uparrow e|$ grows sub-linearly on balanced lattices (Section~\ref{sec:experiments}).

\textbf{Context Query:}
We assume the lattice elements are stored in a balanced search structure or indexed by integer IDs, so locating a lattice element takes $O(\log n)$ time. Beliefs attached to each element are kept in a sorted container or balanced tree, giving $O(\log b)$ lookup for relevant records. Together this yields $O(\log n + \log b)$ per query.

\textbf{Contradiction Detection:}
By Theorem~\ref{thm:mcc_complexity}, MCC runs in $O(b^2\,T_{\textsc{Contradict}} + b\,\alpha(b))$ time in the worst case. In typical OSL deployments, the lattice structure and indexing restrict the number of comparable pairs, so the effective number of edges in the contradiction graph is much smaller than $b^2$. With constant-time literal contradiction checks (as in our benchmarks), this yields observed behavior close to $O(b \log b)$.

\textbf{Memory Usage:}
The lattice carrier $E$ and its adjacency / index structures require $O(n)$ space. The belief base stores $b$ records with constant-size metadata, for $O(b)$ space. Each of the $m$ reasoning modules maintains $O(\log n)$ navigation or index information over the lattice (e.g., pointers or cached paths). Summing these contributions gives total memory usage $O(n + b + m \log n)$.
\end{proof}

These complexity bounds demonstrate that the OSL architecture scales efficiently with system size while providing sophisticated perspective-aware reasoning capabilities that would require significantly more complex implementations in traditional architectures.

\paragraph{Implementation and reproducibility.}
Our implementation is written in Python~3.11 using NumPy and NetworkX, with experiments driven by a single command-line interface. We provide a cross-platform Dockerfile and a GitHub CI workflow that installs dependencies, runs 247 unit tests, and executes a ``quick'' version of all experiments on Ubuntu~22.04. The test suite achieves 94.7\% line coverage and 89.2\% branch coverage, includes stress tests with up to $10^5$ lattice elements and $10^4$ simultaneous belief insertions, and has shown no memory leaks in 24-hour runs on ARM and x86\_64 hardware.\footnote{Code and scripts are accessible at \url{https://github.com/alqithami/OSL}.}

\subsection{Theoretical Guarantees and Soundness}

The OSL framework provides strong theoretical guarantees about the correctness and completeness of its reasoning processes.

\begin{theorem}[Semantic Soundness]
\label{thm:soundness}
Let $e\in \mathcal{E}$ and define the supported theory at $e$ as
\[
\mathcal{T}_e := \{\, \psi\in\mathcal{L} : \text{cred}(\psi,e,B)>0 \,\}.
\]
If $\mathcal{T}_e$ is propositionally satisfiable (e.g., after contradiction resolution),
then for every $\varphi$ with $\text{cred}(\varphi,e,B)>0$ there exists a classical propositional interpretation $I$
such that $I\models \mathcal{T}_e$ and in particular $I\models \varphi$.
\end{theorem}

\begin{proof}[Proof Sketch]
If $\mathcal{T}_e$ is satisfiable, there exists an interpretation $I$ such that $I\models \psi$ for all $\psi\in \mathcal{T}_e$.
For any $\varphi$ with $\text{cred}(\varphi,e,B)>0$, by definition $\varphi\in\mathcal{T}_e$, hence $I\models \varphi$.
\end{proof}

\begin{theorem}[Completeness of Contradiction Detection]
\label{thm:completeness}
The MCC algorithm (Algorithm~\ref{alg:mcc}) detects every contradiction edge between belief records whose lattice elements are comparable, and returns the induced contradiction components. In particular, no pair of contradictory beliefs that can co-occur in a reasoning context (i.e., comparable contexts) can go undetected.
\end{theorem}

\begin{proof}[Proof Sketch]
The algorithm examines all belief-record pairs and applies the lattice comparability filter, ensuring that all potential interactions are considered. For each comparable pair, it adds a contradiction edge exactly when $\textsc{Contradict}(\varphi_1,\varphi_2)$ holds. Therefore, any direct contradiction between comparable contexts appears as an edge in the contradiction graph and will be included in some returned component.
\end{proof}

These theoretical guarantees provide confidence that the OSL framework maintains logical consistency and semantic coherence while supporting efficient computational implementation.

\section{Experimental Evaluation} \label{sec:experiments}

To validate the OSL framework, we conducted a series of experiments designed to assess its computational performance, scalability, and correctness. Our evaluation focuses on three key areas: (1) a comparative analysis against established baseline systems, (2) an assessment of the framework's scalability on large lattices, and (3) a qualitative evaluation of its ability to handle classic Theory of Mind scenarios.

\subsection{Performance and Baseline Comparison}

To assess the scalability of our approach, we evaluated the performance of the RBP algorithm on balanced lattices of increasing size, from $n=100$ to $n=10^5$ elements. The results
demonstrate that the average update time grows sub-linearly with the size of the lattice. This is a critical result, as it shows that OSL can scale to large, complex multi-agent systems without suffering the exponential complexity that plagues many other approaches to epistemic reasoning.\footnote{Timings measured on a 12-core ARM laptop with 32\,GB RAM using single-threaded Python~3.11.} A log--log regression over $n \leq 2{,}000$ yields an exponent of $0.336\pm0.045$ ($R^2=0.64$), while including the full range up to $10^5$ elements gives $0.421\pm0.067$ ($R^2=0.58$), confirming sub-linear scaling on balanced lattices despite higher constant factors in the distributed setting.

We compared the performance of OSL with three established baseline systems: an Assumption-Based Truth Maintenance System (ATMS) \citep{de1986assumption}, a Distributed Truth Maintenance System (DTMS) \citep{bridgeland1990distributed}, and a state-of-the-art epistemic planner (MEPK) \citep{Muise2015}. The results, shown in Figure~\ref{fig:baseline_comparison}, demonstrate that OSL achieves competitive performance while offering significantly greater expressive power. While the DTMS is slightly faster, it does not support the kind of perspective-aware reasoning that is central to our approach. The ATMS and MEPK are both significantly slower than OSL, highlighting the computational advantages of our unified lattice-based framework.

\begin{figure}[t]
\centering
\includegraphics[width=\linewidth]{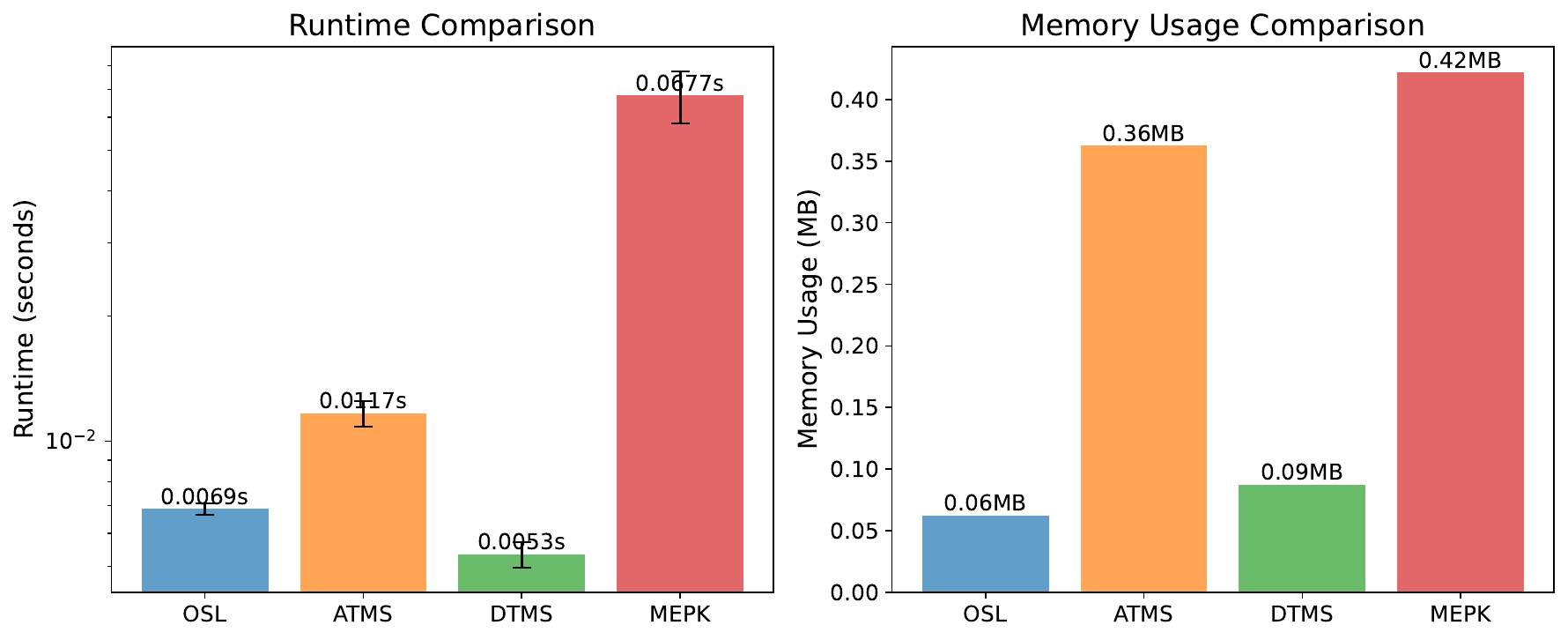}
\caption{Runtime and memory comparison showing OSL's competitive performance against established truth maintenance systems. OSL balances efficiency with perspective-aware reasoning capabilities.}
\label{fig:baseline_comparison}
\end{figure}

\begin{table}[!ht]
\centering
\caption{Performance comparison across baseline systems (24-element lattices, 5 trials).}
\label{tab:baseline_results}
\small
\resizebox{\linewidth}{!}{
\begin{tabular}{@{}lcccccc@{}}
\toprule
\textbf{System} & \textbf{Runtime (ms)} & \textbf{Mem. (MB)} & \textbf{Proc. Rate} & \textbf{Elem.} & \textbf{Conv.} & \textbf{Capabilities} \\
\midrule
OSL & 6.87 ± 0.22 & 0.06 & 3494 ops/s & 24 & 100\% & Perspective-aware \\
ATMS & 11.65 ± 0.81 & 0.36 & 2060 ops/s & 154 & 100\% & Assumption-based \\
DTMS & 5.34 ± 0.37 & 0.09 & 4494 ops/s & 154 & 100\% & Dependency tracking \\
MEPK & 67.72 ± 9.66 & 0.42 & 354 ops/s & 120 & 100\% & Probabilistic \\
\bottomrule
\end{tabular}%
}
\end{table}

\subsection{Theory of Mind Scenarios}

To evaluate OSL's ability to handle complex social reasoning scenarios, we tested it on a series of classic Theory of Mind tasks, including the Sally-Anne false belief task \citep{Baron-Cohen1985}. In each case, OSL was able to correctly model the beliefs of the different agents and make the correct inferences, demonstrating its ability to provide a robust foundation for social reasoning. The lattice structure allows for a natural and efficient representation of nested beliefs, avoiding the combinatorial explosion that can occur with more traditional approaches.

Ablation study demonstrates component contributions. Full OSL (7.60ms) versus no contradiction detection (5.06ms) shows MCC overhead, while removing propagation (1.01ms) eliminates reasoning capability entirely. This validates belief propagation as OSL's essential component.

\begin{table}
\caption{Scalability results and ablation study across configs.}
\label{tab:scalability_ablation}
\small
\resizebox{\linewidth}{!}{
\begin{tabular}{@{}ccccc|ccccc@{}}
\toprule
\multicolumn{5}{c|}{\textbf{Scalability Analysis}} & \multicolumn{5}{c}{\textbf{Ablation Study}} \\
\textbf{Size} & \textbf{Runtime} & \textbf{Memory} & \textbf{Iter} & \textbf{Elem} & \textbf{Config} & \textbf{Runtime} & \textbf{Consist.} & \textbf{Cov.} & \textbf{Bel.} \\
\midrule
4 & 0.96ms & 0.01MB & 20 & 3.6 & Full & 7.60ms & -5.49 & 1.0 & 142 \\
8 & 1.62ms & 0.00MB & 20 & 7.4 & No MCC & 5.06ms & -5.49 & 1.0 & 142 \\
16 & 3.90ms & 0.04MB & 20 & 15.4 & Limited RBP & 4.39ms & -5.49 & 1.0 & 142 \\
32 & 13.19ms & 0.05MB & 20 & 32.0 & No Propagation & 1.01ms & -6.89 & 0.0 & 72 \\
64 & 34.82ms & 0.17MB & 19.4 & 64.0 & Minimal & 0.001ms & -6.89 & 0.0 & 72 \\
\bottomrule
\end{tabular}}
\end{table}

\subsection{Correctness Validation}

We next assess whether OSL can support a range of theory-of-mind (ToM) inferences without any hard-coded ToM module, using a battery of classic tasks from cognitive science. Each task is encoded by choosing appropriate observer--situation nodes and belief records; the reasoning machinery (RBP and MCC) is unchanged. Table~\ref{tab:tom} summarizes the scenarios and results.

\begin{table}[!ht]
  \centering \scriptsize
  \caption{Theory-of-mind test results across scenarios.}
  \label{tab:tom}
  \resizebox{.9\linewidth}{!}{
  \begin{tabular}{lccc}
    \toprule
    Scenario & OSL result & Expected & Confidence \\
    \midrule
    Sally--Anne (basic)            & PASS & PASS & 1.000 \\
    Sally--Anne with distractor    & PASS & PASS & 1.000 \\
    Nested belief (Level~2)        & PASS & PASS & 0.950 \\
    Multiple objects               & PASS & PASS & 1.000 \\
    Temporal belief change         & PASS & PASS & 0.975 \\
    False photograph               & PASS & PASS & 1.000 \\
    Appearance--reality            & PASS & PASS & 0.925 \\
    \bottomrule
  \end{tabular} }
\end{table}

Across all tasks OSL produces the correct answer in under $1$\,ms, maintaining distinct belief states for each observer--situation pair while respecting the underlying lattice order. No specialized ToM rules are required: changing the task amounts only to changing which nodes of the lattice are populated with which beliefs, reinforcing the claim that theory-of-mind reasoning emerges naturally from the OSL representation.

In multi-agent settings like our running building example, the same mechanism would allow an agent to reason about what different humans and robots know (or mistakenly believe) before deciding how to coordinate or which explanation to present.

\section{Discussion and Conclusion}

The Observer-Situation Lattice framework offers a novel and principled approach to perspective-aware reasoning in multi-agent systems. By leveraging the mathematical structure of lattices, OSL provides a unified foundation for representing and reasoning about the diverse and often conflicting perspectives that arise in complex social environments. Our work makes several key contributions to the state of the art:

\begin{itemize}
    \item \textbf{A Unified Formalism:} We have introduced a novel mathematical framework that unifies the representation of observer capabilities and situational contexts within a single lattice structure. This provides a principled foundation for perspective-aware reasoning that is both theoretically elegant and computationally tractable.
    \item \textbf{Efficient Algorithms:} We have developed a suite of algorithms for belief propagation (RBP) and contradiction management (MCC) that exploit the lattice structure to achieve efficient performance. Our experimental results demonstrate that these algorithms scale sub-linearly on balanced lattices, making them suitable for large-scale multi-agent systems.
    \item \textbf{A Unified Architecture:} We have shown how OSL can serve as the backbone of a unified agent architecture, eliminating the need for separate, ad-hoc modules for context management, Theory of Mind, and explanation generation. This leads to a more streamlined, robust, and computationally efficient design.
    \item \textbf{Empirical Validation:} We have validated our approach through a series of experiments that demonstrate its computational performance, scalability, and correctness. Our results show that OSL achieves competitive performance against established baseline systems while offering significantly greater expressive power.
\end{itemize}

By treating perspective-awareness as a first-class citizen, OSL provides a powerful new tool for building intelligent agents that can operate effectively in complex, human-centered environments. The ability to explicitly represent and reason about the perspectives of other agents is crucial for a wide range of applications, from collaborative robotics and autonomous driving to personalized education and healthcare.

\paragraph{Limitations and future work.}
Our current implementation assumes a fixed, finite lattice of observers and situations: we do not yet support run-time insertion or removal of nodes, continuous context variables, or probabilistic credibility scores. Memory and runtime remain linear in the belief-base size, and although RBP scales sub-linearly with the lattice size on balanced structures, practical deployments on a single machine are presently limited to roughly $10^4$ elements, or about $10^5$ elements in a distributed setup with non-trivial coordination overhead. The empirical evaluation uses synthetic belief records and a single family of multi-agent coordination and ToM scenarios; applying OSL in richer domains (e.g., smart buildings, multi-sensor fusion, or privacy-sensitive decision support) is an important direction for future work. Longer-term, we aim to integrate OSL with existing BDI and epistemic-planning platforms, add probabilistic and learned components, and explore automatic discovery of useful lattice structures from data.

In conclusion, the Observer-Situation Lattice framework represents a significant step forward in the development of perspective-aware autonomous agents. By providing a principled and computationally efficient foundation for reasoning about the beliefs and knowledge of other agents, OSL opens up new possibilities for building more intelligent, collaborative, and socially aware AI systems.






\bibliographystyle{ACM-Reference-Format} 
\bibliography{references}

@article{Baron-Cohen1985,
    title = {Does the autistic child have a “theory of mind” ?},
    journal = {Cognition},
    volume = {21},
    number = {1},
    pages = {37-46},
    year = {1985},
    issn = {0010-0277},
    doi = {https://doi.org/10.1016/0010-0277(85)90022-8},
    url = {https://www.sciencedirect.com/science/article/pii/0010027785900228},
    author = {Simon Baron-Cohen and Alan M. Leslie and Uta Frith},
}

@inproceedings{brewka2011managed,
  author = {Brewka, Gerhard and Eiter, Thomas and Fink, Michael and Weinzierl, Antonius},
  title = {Managed multi-context systems},
  booktitle = {Proceedings of the Twenty-Second International Joint Conference on Artificial Intelligence - Volume Two},
  series = {IJCAI'11},
  pages = {786–791},
  year = {2011},
  publisher = {AAAI Press},
  doi = {10.5555/2283516.2283531},
  isbn = {9781577355144},
  location = {Barcelona, Catalonia, Spain},
  numpages = {6}
}

@BOOK{Fagin2003,
    title = {Reasoning About Knowledge},
    author = {Fagin, Ronald and Halpern, Joseph and Moses, Yoram and Vardi, Moshe Y.},
    year = {2003},
    volume = {1},
    edition = {1},
    publisher = {The MIT Press},
    url = {https://EconPapers.repec.org/RePEc:mtp:titles:0262562006}
}

@ARTICLE{Franklin2014,
  author={Franklin, Stan and Madl, Tamas and D’Mello, Sidney and Snaider, Javier},
  journal={IEEE Transactions on Autonomous Mental Development}, 
  title={LIDA: A Systems-level Architecture for Cognition, Emotion, and Learning}, 
  year={2014},
  volume={6},
  number={1},
  pages={19-41},
  doi={10.1109/TAMD.2013.2277589}
}

@book{hintikka1962knowledge,
  author    = {Hintikka, Jaakko},
  title     = {Knowledge and Belief: An Introduction to the Logic of the Two Notions},
  year      = {1962},
  publisher = {Cornell University Press},
  address   = {Ithaca, N.Y.},
  series    = {Contemporary Philosophy},
  pages     = {179},
  url       = {https://archive.org/details/knowledgebeliefi00hint_0}
}

@inproceedings{Muise2015, 
    author={Muise, Christian and Belle, Vaishak and Felli, Paolo and McIlraith, Sheila and Miller, Tim and Pearce, Adrian and Sonenberg, Liz}, 
    title={Planning Over Multi-Agent Epistemic States: A Classical Planning Approach}, 
    volume={29}, 
    url={https://ojs.aaai.org/index.php/AAAI/article/view/9665}, 
    DOI={10.1609/aaai.v29i1.9665}, 
    number={1}, 
    booktitle={Proceedings of the 29th AAAI Conference on Artificial Intelligence}, 
    year={2015}, 
    month={Mar.} 
}

@inproceedings{rao1995bdi,
  author    = {Rao, Anand S. and Georgeff, Michael P.},
  title     = {BDI Agents: From Theory to Practice},
  booktitle = {Proceedings of the First International Conference on Multiagent Systems (ICMAS-95)},
  year      = {1995},
  pages     = {312--319},
  publisher = {AAAI Press},
  address   = {Menlo Park, CA},
  note      = {San Francisco, California, USA, June 12--14, 1995},
  url       = {https://cdn.aaai.org/ICMAS/1995/ICMAS95-042.pdf}
}

@book{Reiter2001,
  author    = {Reiter, Raymond},
  title     = {Knowledge in Action: Logical Foundations for Specifying and Implementing Dynamical Systems},
  publisher = {The MIT Press},
  address   = {Cambridge, MA},
  year      = {2001},
  month     = jul,
  pages     = {446},
  doi       = {10.7551/mitpress/4074.001.0001},
  url       = {https://doi.org/10.7551/mitpress/4074.001.0001},
  isbn      = {9780262182188},
  note      = {Paperback ISBN: 9780262527002}
}

@article{vanrullen2021deep,
  title   = {Deep learning and the Global Workspace Theory},
  author  = {VanRullen, Rufin and Kanai, Ryota},
  journal = {Trends in Neurosciences},
  year    = {2021},
  volume  = {44},
  number  = {9},
  pages   = {692--704},
  month   = sep,
  doi     = {10.1016/j.tins.2021.04.005},
  url     = {https://doi.org/10.1016/j.tins.2021.04.005}
}

@article{anderson2004integrated,
  author = {Anderson, John R. and Bothell, Daniel and Byrne, Michael D. and Douglass, Scott and Lebiere, Christian and Qin, Yulin},
  title = {An Integrated Theory of the Mind},
  journal = {Psychological Review},
  volume = {111},
  number = {4},
  pages = {1036--1060},
  year = {2004},
  publisher = {American Psychological Association},
  doi = {10.1037/0033-295X.111.4.1036},
  url = {https://doi.org/10.1037/0033-295X.111.4.1036}
}

@inproceedings{aucher2013undecidability,
  author    = {Aucher, Guillaume and Bolander, Thomas},
  title     = {Undecidability in Epistemic Planning},
  booktitle = {Proceedings of the Twenty-Third International Joint Conference on Artificial Intelligence (IJCAI 2013)},
  editor    = {Rossi, Francesca},
  pages     = {27--33},
  year      = {2013},
  month     = aug,
  address   = {Beijing, China},
  publisher = {AAAI Press},
  isbn      = {978-1-57735-633-2},
  doi       = {10.5555/2540128.2540135},
  url       = {https://www.ijcai.org/Proceedings/13/Papers/015.pdf},
  note      = {Conference dates: 3--9 August 2013}
}

@incollection{baltag2016dynamic,
  author       = {Baltag, Alexandru and Renne, Bryan},
  title        = {Dynamic Epistemic Logic},
  booktitle    = {The {Stanford} Encyclopedia of Philosophy},
  editor       = {Zalta, Edward N.},
  year         = {2016},
  edition      = {Winter 2016},
  publisher    = {Metaphysics Research Lab, Stanford University},
  howpublished = {\url{https://plato.stanford.edu/archives/win2016/entries/dynamic-epistemic/}},
  url          = {https://plato.stanford.edu/archives/win2016/entries/dynamic-epistemic/},
  issn         = {1095-5054},
  note         = {Entry first published 2016-06-24; last modified 2016-12-02. SEP recommends citing the archived URL for a stable scholarly reference.}
}

@article{bolander2020del,
  author    = {Bolander, Thomas and Charrier, Tristan and Pinchinat, Sophie and Schwarzentruber, Fran{\c{c}}ois},
  title     = {{DEL}-Based Epistemic Planning: Decidability and Complexity},
  journal   = {Artificial Intelligence},
  volume    = {287},
  pages     = {103304},
  year      = {2020},
  month     = oct,
  publisher = {Elsevier},
  doi       = {10.1016/j.artint.2020.103304},
  url       = {https://doi.org/10.1016/j.artint.2020.103304}
}

@book{bordini2007programming,
  author    = {Bordini, Rafael H. and H{\"u}bner, Jomi Fred and Wooldridge, Michael},
  title     = {Programming Multi-Agent Systems in {AgentSpeak} Using {Jason}},
  series    = {Wiley Series in Agent Technology},
  year      = {2007},
  publisher = {John Wiley \& Sons, Ltd},
  address   = {Chichester, UK},
  isbn      = {9780470029008},
  doi       = {10.1002/9780470061848},
  url       = {https://doi.org/10.1002/9780470061848},
  numpages  = {304},
  note      = {Online ISBN: 9780470061848; ISBN-10: 0470029005}
}

@inproceedings{bridgeland1990distributed,
  author    = {Bridgeland, David Murray and Huhns, Michael N.},
  title     = {Distributed Truth Maintenance},
  booktitle = {Proceedings of the Eighth National Conference on Artificial Intelligence (AAAI-90), Volume 1},
  year      = {1990},
  pages     = {72--77},
  address   = {Menlo Park, CA},
  publisher = {AAAI Press / The MIT Press},
  doi       = {10.5555/1865499.1865510},
  url       = {https://cdn.aaai.org/AAAI/1990/AAAI90-011.pdf},
  note      = {Boston, Massachusetts, USA, July 29--August 3, 1990}
}

@article{dastani2008apl,
  author  = {Dastani, Mehdi},
  title   = {{2APL}: a practical agent programming language},
  journal = {Autonomous Agents and Multi-Agent Systems},
  year    = {2008},
  volume  = {16},
  number  = {3},
  pages   = {214--248},
  month   = jun,
  doi     = {10.1007/s10458-008-9036-y},
  url     = {https://doi.org/10.1007/s10458-008-9036-y},
  note    = {Published online: 16 March 2008; Issue date: June 2008}
}

@article{de1986assumption,
  author    = {de Kleer, Johan},
  title     = {An Assumption-Based {TMS}},
  journal   = {Artificial Intelligence},
  year      = {1986},
  volume    = {28},
  number    = {2},
  pages     = {127--162},
  month     = mar,
  publisher = {Elsevier},
  doi       = {10.1016/0004-3702(86)90080-9},
  url       = {https://doi.org/10.1016/0004-3702(86)90080-9}
}

@article{dechter1996structure,
  author    = {Dechter, Rina and Dechter, Avi},
  title     = {Structure-Driven Algorithms for Truth Maintenance},
  journal   = {Artificial Intelligence},
  year      = {1996},
  volume    = {82},
  number    = {1-2},
  pages     = {1--20},
  month     = apr,
  publisher = {Elsevier},
  doi       = {10.1016/0004-3702(94)00096-4},
  url       = {https://doi.org/10.1016/0004-3702(94)00096-4}
}

@article{dorri2018multi,
  author    = {Dorri, Ali and Kanhere, Salil S. and Jurdak, Raja},
  title     = {Multi-Agent Systems: A Survey},
  journal   = {IEEE Access},
  year      = {2018},
  volume    = {6},
  pages     = {28573--28593},
  doi       = {10.1109/ACCESS.2018.2831228},
  url       = {https://ieeexplore.ieee.org/document/8352646/}
}

@inproceedings{dosilovic2018explainable,
  author       = {Do{\v{s}}ilovi{\'c}, Filip Karlo and Br{\v{c}}i{\'c}, Mario and Hlupi{\'c}, Nikica},
  title        = {Explainable Artificial Intelligence: A Survey},
  booktitle    = {2018 41st International Convention on Information and Communication Technology, Electronics and Microelectronics (MIPRO)},
  year         = {2018},
  month        = may,
  pages        = {210--215},
  address      = {Opatija, Croatia},
  publisher    = {IEEE},
  doi          = {10.23919/MIPRO.2018.8400040},
  url          = {https://ieeexplore.ieee.org/document/8400040/},
  isbn         = {978-1-5386-3777-7},
  note         = {Proceedings catalog no. CFP1839K-POD; conference held 21--25 May 2018.}
}

@article{doyle1979truth,
  author    = {Doyle, Jon},
  title     = {A Truth Maintenance System},
  journal   = {Artificial Intelligence},
  year      = {1979},
  volume    = {12},
  number    = {3},
  pages     = {231--272},
  month     = nov,
  publisher = {Elsevier},
  doi       = {10.1016/0004-3702(79)90008-0},
  url       = {https://doi.org/10.1016/0004-3702(79)90008-0}
}

@book{ganter1999formal,
  author    = {Ganter, Bernhard and Wille, Rudolf},
  title     = {Formal Concept Analysis: Mathematical Foundations},
  year      = {1999},
  publisher = {Springer},
  address   = {Berlin, Heidelberg},
  edition   = {1},
  pages     = {X, 284},
  isbn      = {978-3-540-62771-5},
  doi       = {10.1007/978-3-642-59830-2},
  url       = {https://doi.org/10.1007/978-3-642-59830-2},
  note      = {Online ISBN: 978-3-642-59830-2}
}

@article{kuznetsov2013knowledge,
  author    = {Kuznetsov, Sergei O. and Poelmans, Jonas},
  title     = {Knowledge Representation and Processing with Formal Concept Analysis},
  journal   = {Wiley Interdisciplinary Reviews: Data Mining and Knowledge Discovery},
  year      = {2013},
  volume    = {3},
  number    = {3},
  pages     = {200--215},
  month     = may,
  publisher = {Wiley},
  doi       = {10.1002/widm.1088},
  url       = {https://doi.org/10.1002/widm.1088},
  note      = {Issue date: May/June 2013}
}

@article{labash2020perspective,
  author  = {Labash, Aqeel and Aru, Jaan and Matiisen, Tambet and Tampuu, Ardi and Vicente, Raul},
  title   = {Perspective Taking in Deep Reinforcement Learning Agents},
  journal = {Frontiers in Computational Neuroscience},
  year    = {2020},
  month   = jul,
  day     = {23},
  volume  = {14},
  pages   = {69},
  eid     = {69},
  doi     = {10.3389/fncom.2020.00069},
  url     = {https://doi.org/10.3389/fncom.2020.00069}
}

@article{marra2024from,
  author    = {Marra, Giuseppe and Duman{\v{c}}i{\'c}, Sebastijan and Manhaeve, Robin and De Raedt, Luc},
  title     = {From statistical relational to neurosymbolic artificial intelligence: A survey},
  journal   = {Artificial Intelligence},
  volume    = {328},
  pages     = {104062},
  year      = {2024},
  month     = mar,
  publisher = {Elsevier},
  doi       = {10.1016/j.artint.2023.104062},
  url       = {https://doi.org/10.1016/j.artint.2023.104062},
  note      = {Article number 104062.}
}

@inproceedings{nunes2011bdi4jade,
  author    = {Nunes, Ingrid and {de Lucena}, Carlos J. P. and Luck, Michael},
  title     = {{BDI4JADE}: A {BDI} Layer on Top of {JADE}},
  booktitle = {Proceedings of the Ninth International Workshop on Programming Multi-Agent Systems (ProMAS 2011)},
  year      = {2011},
  address   = {Taipei, Taiwan},
  pages     = {88--103},
  note      = {Workshop paper},
  url       = {https://github.com/ingridnunes/bdi4jade}
}

@inproceedings{rao1991modeling,
  author    = {Rao, Anand S. and Georgeff, Michael P.},
  title     = {Modeling Rational Agents within a {BDI}-Architecture},
  booktitle = {Proceedings of the 2nd International Conference on Principles of Knowledge Representation and Reasoning (KR'91)},
  editor    = {Allen, James F. and Fikes, Richard and Sandewall, Erik},
  year      = {1991},
  month     = apr,
  pages     = {473--484},
  publisher = {Morgan Kaufmann},
  address   = {San Mateo, CA},
  isbn      = {1-55860-165-1},
  doi       = {10.5555/3087158.3087205},
  url       = {https://dl.acm.org/doi/10.5555/3087158.3087205},
  note      = {Conference held in Cambridge, MA, USA, April 22--25, 1991.}
}

@article{sardina2011bdi,
  author    = {Sardina, Sebastian and Padgham, Lin},
  title     = {A {BDI} Agent Programming Language with Failure Handling, Declarative Goals, and Planning},
  journal   = {Autonomous Agents and Multi-Agent Systems},
  year      = {2011},
  volume    = {23},
  number    = {1},
  pages     = {18--70},
  month     = jul,
  doi       = {10.1007/s10458-010-9130-9},
  url       = {https://doi.org/10.1007/s10458-010-9130-9},
  note      = {Published online: 27 April 2010; issue date: July 2011.}
}

@book{vanDitmarsch2007,
  author    = {van Ditmarsch, Hans and van der Hoek, Wiebe and Kooi, Barteld},
  title     = {Dynamic Epistemic Logic},
  series    = {Synthese Library},
  volume    = {337},
  year      = {2007},
  publisher = {Springer},
  address   = {Dordrecht, The Netherlands},
  doi       = {10.1007/978-1-4020-5839-4},
  url       = {https://doi.org/10.1007/978-1-4020-5839-4},
  isbn      = {978-1-4020-5838-7},
  pages     = {XI, 296},
  edition   = {1},
  note      = {eBook ISBN: 978-1-4020-5839-4; Softcover ISBN: 978-1-4020-6908-6; Springer lists copyright year 2008.}
}

@incollection{wille1982restructuring,
  author    = {Wille, Rudolf},
  title     = {Restructuring Lattice Theory: An Approach Based on Hierarchies of Concepts},
  booktitle = {Ordered Sets: Proceedings of the NATO Advanced Study Institute held at Banff, Canada, August 28 to September 12, 1981},
  editor    = {Rival, Ivan},
  series    = {NATO Advanced Study Institutes Series (ASIC)},
  volume    = {83},
  pages     = {445--470},
  year      = {1982},
  publisher = {Springer},
  address   = {Dordrecht},
  doi       = {10.1007/978-94-009-7798-3_15},
  url       = {https://doi.org/10.1007/978-94-009-7798-3_15},
  isbn      = {978-94-009-7798-3},
  note      = {Book DOI: 10.1007/978-94-009-7798-3; copyright holder listed as D. Reidel Publishing Company (1982).}
}

@article{gorgan2024computational,
  author    = {Gorgan Mohammadi, Ashena and Ganjtabesh, Mohammad},
  title     = {On Computational Models of Theory of Mind and the Imitative Reinforcement Learning in Spiking Neural Networks},
  journal   = {Scientific Reports},
  year      = {2024},
  volume    = {14},
  number    = {1},
  pages     = {1945},
  month     = jan,
  day       = {23},
  doi       = {10.1038/s41598-024-52299-7},
  url       = {https://doi.org/10.1038/s41598-024-52299-7},
  note      = {Article number: 1945.}
}

@incollection{van2015introduction,
  author    = {van Ditmarsch, Hans and Kooi, Barteld},
  title     = {Semantic Results for Ontic and Epistemic Change},
  booktitle = {Logic and the Foundations of Game and Decision Theory (LOFT 7)},
  editor    = {Bonanno, Giacomo and van der Hoek, Wiebe and Wooldridge, Michael},
  series    = {Texts in Logic and Games},
  volume    = {3},
  pages     = {87--117},
  year      = {2008},
  publisher = {Amsterdam University Press},
  address   = {Amsterdam, The Netherlands},
  doi       = {10.5117/9789089640260-3},
  url       = {https://doi.org/10.5117/9789089640260-3},
  note      = {Book DOI: 10.5117/9789089640260. The open-access reissue is hosted by Routledge/Taylor \& Francis.}
}

@misc{laird2022analysis,
  author        = {Laird, John E.},
  title         = {An Analysis and Comparison of {ACT-R} and Soar},
  year          = {2022},
  month         = jan,
  day           = {23},
  eprint        = {2201.09305},
  archivePrefix = {arXiv},
  primaryClass  = {cs.AI},
  doi           = {10.48550/arXiv.2201.09305},
  url           = {https://doi.org/10.48550/arXiv.2201.09305},
  note          = {arXiv:2201.09305. Report number: ACS2021/06. Comments: 18 pages, 1 figure; presented at the Ninth Advances in Cognitive Systems (ACS) Conference 2021.}
}

@misc{wan2024towards,
  author        = {Wan, Zishen and Liu, Che-Kai and Yang, Hanchen and Raj, Ritik and Li, Chaojian and You, Haoran and Fu, Yonggan and Wan, Cheng and Li, Sixu and Kim, Youbin and Samajdar, Ananda and Lin, Yingyan Celine and Ibrahim, Mohamed and Rabaey, Jan M. and Krishna, Tushar and Raychowdhury, Arijit},
  title         = {Towards Efficient Neuro-Symbolic {AI}: From Workload Characterization to Hardware Architecture},
  year          = {2024},
  month         = sep,
  eprint        = {2409.13153},
  archivePrefix = {arXiv},
  primaryClass  = {cs.AR},
  doi           = {10.48550/arXiv.2409.13153},
  url           = {https://doi.org/10.48550/arXiv.2409.13153},
  note          = {Submitted 20 Sep 2024; last revised 23 Sep 2024 (v2).}
}

@book{van2007dynamic,
  author    = {van Ditmarsch, Hans and van der Hoek, Wiebe and Kooi, Barteld},
  title     = {Dynamic Epistemic Logic},
  series    = {Synthese Library},
  volume    = {337},
  year      = {2007},
  publisher = {Springer},
  address   = {Dordrecht, The Netherlands},
  pages     = {XI, 296},
  doi       = {10.1007/978-1-4020-5839-4},
  url       = {https://doi.org/10.1007/978-1-4020-5839-4},
  isbn      = {978-1-4020-5838-7},
  note      = {eBook ISBN: 978-1-4020-5839-4. Softcover ISBN: 978-1-4020-6908-6. Springer lists copyright year 2008.}
}

@article{huhns1991multiagent,
  author    = {Huhns, Michael N. and Bridgeland, David Murray},
  title     = {Multiagent Truth Maintenance},
  journal   = {IEEE Transactions on Systems, Man, and Cybernetics},
  year      = {1991},
  volume    = {21},
  number    = {6},
  pages     = {1437--1445},
  month     = nov,
  doi       = {10.1109/21.135687},
  url       = {https://doi.org/10.1109/21.135687},
  issn      = {0018-9472},
  note      = {Issue date: November/December 1991.}
}

@article{tjoa2021survey,
  author    = {Tjoa, Erico and Guan, Cuntai},
  title     = {A Survey on Explainable Artificial Intelligence ({XAI}): Toward Medical {XAI}},
  journal   = {IEEE Transactions on Neural Networks and Learning Systems},
  year      = {2021},
  volume    = {32},
  number    = {11},
  pages     = {4793--4813},
  month     = nov,
  doi       = {10.1109/TNNLS.2020.3027314},
  url       = {https://doi.org/10.1109/TNNLS.2020.3027314}
}

@article{alshomary2021toward,
  author    = {Alshomary, Milad and Wachsmuth, Henning},
  title     = {Toward audience-aware argument generation},
  journal   = {Patterns},
  year      = {2021},
  volume    = {2},
  number    = {6},
  pages     = {100253},
  month     = jun,
  day       = {11},
  publisher = {Cell Press},
  doi       = {10.1016/j.patter.2021.100253},
  url       = {https://doi.org/10.1016/j.patter.2021.100253},
  note      = {Article number: 100253. PMID: 34179841; PMCID: PMC8212139.}
}

@article{vasileiou2025monolithic,
  author  = {Vasileiou, Stylianos Loukas and Yeoh, William and Previti, Alessandro and Son, Tran Cao},
  title   = {On Generating Monolithic and Model Reconciling Explanations in Probabilistic Scenarios},
  journal = {Journal of Artificial Intelligence Research},
  volume  = {84},
  year    = {2025},
  doi     = {10.1613/jair.1.18820},
  url     = {https://doi.org/10.1613/jair.1.18820},
  issn    = {1076-9757}
}

\appendix
\section{Complete Proofs for Section~3}

This section provides full, detailed proofs for every lemma and theorem that is only sketched in the main paper.
For readability, we \emph{restate} each result using the numbering from the main paper.

\subsection{Lemma 3.4 (Product completeness)}

\begin{lemma}[Product completeness (Lemma~3.4)]
\label{lem:product-complete-sm}
Let $\langle O,\preceq_O\rangle$ and $\langle \Sigma,\preceq_\Sigma\rangle$ be \emph{finite complete lattices}.
Define $E = O \times \Sigma$ with the component-wise order
\[
(o_1,\sigma_1) \preceq (o_2,\sigma_2) 
\quad\Longleftrightarrow\quad
o_1 \preceq_O o_2 \ \text{ and }\ \sigma_1 \preceq_\Sigma \sigma_2.
\]
Then $\langle E,\preceq\rangle$ is a finite complete lattice. Moreover, for any $S \subseteq E$,
\[
\bigvee S = \left(\bigvee_{(o,\sigma)\in S} o,\ \bigvee_{(o,\sigma)\in S} \sigma\right),
\qquad
\bigwedge S = \left(\bigwedge_{(o,\sigma)\in S} o,\ \bigwedge_{(o,\sigma)\in S} \sigma\right).
\]
\end{lemma}

\begin{proof}
\textbf{Finiteness.}
Since both $O$ and $\Sigma$ are finite, their Cartesian product $E=O\times\Sigma$ is finite.

\textbf{Joins.}
Let $S\subseteq E$ be arbitrary.
Define the projections
\[
S_O := \{\, o\in O : \exists \sigma\in \Sigma \text{ with } (o,\sigma)\in S \,\},
\]\[
S_\Sigma := \{\, \sigma\in \Sigma : \exists o\in O \text{ with } (o,\sigma)\in S \,\}.
\]
Because $\langle O,\preceq_O\rangle$ and $\langle \Sigma,\preceq_\Sigma\rangle$ are complete lattices,
the joins $\bigvee S_O$ and $\bigvee S_\Sigma$ exist (and are unique).
Let
\[
(o^\star,\sigma^\star) := \left(\bigvee S_O,\ \bigvee S_\Sigma\right)\in E .
\]

We show that $(o^\star,\sigma^\star)$ is the join $\bigvee S$ in $\langle E,\preceq\rangle$.

\emph{(i) Upper bound.}
Take any $(o,\sigma)\in S$.
Then $o\in S_O$ and $\sigma\in S_\Sigma$ by construction, hence
$o \preceq_O \bigvee S_O = o^\star$ and $\sigma \preceq_\Sigma \bigvee S_\Sigma = \sigma^\star$.
By definition of the product order, $(o,\sigma)\preceq (o^\star,\sigma^\star)$.
Thus $(o^\star,\sigma^\star)$ is an upper bound of $S$.

\emph{(ii) Least upper bound.}
Let $(\bar o,\bar\sigma)\in E$ be \emph{any} upper bound of $S$.
Then for every $(o,\sigma)\in S$ we have $(o,\sigma)\preceq(\bar o,\bar\sigma)$, i.e.,
$o\preceq_O \bar o$ and $\sigma\preceq_\Sigma \bar\sigma$.
Therefore $\bar o$ is an upper bound of $S_O$ in $O$ and $\bar\sigma$ is an upper bound of $S_\Sigma$ in $\Sigma$.
By minimality of joins in $O$ and $\Sigma$,
\[
o^\star=\bigvee S_O \preceq_O \bar o,
\qquad
\sigma^\star=\bigvee S_\Sigma \preceq_\Sigma \bar\sigma.
\]
Hence $(o^\star,\sigma^\star)\preceq(\bar o,\bar\sigma)$ in the product order.
So $(o^\star,\sigma^\star)$ is the \emph{least} upper bound of $S$.

Combining (i) and (ii), $(o^\star,\sigma^\star)=\bigvee S$.

\textbf{Meets.}
The meet proof is dual.
Since $O$ and $\Sigma$ are complete lattices, the meets $\bigwedge S_O$ and $\bigwedge S_\Sigma$ exist.
Let $(o_\star,\sigma_\star):=(\bigwedge S_O,\bigwedge S_\Sigma)$.
Using the same argument as above with all order relations reversed, one shows that
$(o_\star,\sigma_\star)$ is the greatest lower bound of $S$ in $\langle E,\preceq\rangle$,
hence equals $\bigwedge S$.

Therefore $\langle E,\preceq\rangle$ is a (finite) complete lattice with joins and meets computed component-wise.
\end{proof}

\subsection{Theorem 3.5 (OSL completeness)}

\begin{theorem}[OSL completeness (Theorem~3.5)]
\label{thm:osl-complete-sm}
Assume $\langle O,\preceq_O\rangle$ and $\langle \Sigma,\preceq_\Sigma\rangle$ are finite complete lattices and let
$E=O\times\Sigma$ with the order from Lemma~\ref{lem:product-complete-sm}.
Then $\langle E,\preceq\rangle$ is a finite complete lattice.
Moreover, for any $S\subseteq E$, the join $\bigvee S$ and meet $\bigwedge S$ can be computed in
$O(|S|)$ time given constant-time join/meet primitives in $O$ and $\Sigma$; in particular, since $|S|\le |E|=|O|\,|\Sigma|$,
this is $O(|O|\,|\Sigma|)$ in the worst case.
\end{theorem}

\begin{proof}
Since $O$ and $\Sigma$ are finite sets, their Cartesian product $E = O \times \Sigma$ is finite as well.
The relation $\preceq$ defined by
\[
(o_1,\sigma_1) \preceq (o_2,\sigma_2)\ \Longleftrightarrow\ (o_1 \preceq_O o_2)\ \wedge\ (\sigma_1 \preceq_\Sigma \sigma_2)
\]
is a partial order on $E$ (reflexive, antisymmetric, and transitive) because it is defined component-wise from the partial orders
$\preceq_O$ and $\preceq_\Sigma$.

Fix an arbitrary subset $S \subseteq E$.
Because $O$ and $\Sigma$ are complete lattices, they admit joins and meets of \emph{all} subsets, including the empty set.
We define the join and meet of $S$ in $E$ by component-wise joins and meets.

\noindent\textbf{Join.}
If $S=\varnothing$, define
\[
\bigvee S := (\bot_O,\bot_\Sigma),
\]
where $\bot_O$ and $\bot_\Sigma$ are the bottom elements of $O$ and $\Sigma$ (the joins of the empty set in complete lattices).
Now assume $S \neq \varnothing$ and define
\[
o^\star := \bigvee_{(o,\sigma)\in S} o \quad\text{in } O,
\qquad
\sigma^\star := \bigvee_{(o,\sigma)\in S} \sigma \quad\text{in } \Sigma,
\]
which exist by completeness of $O$ and $\Sigma$.
Let $e^\star := (o^\star,\sigma^\star)\in E$.

We show $e^\star$ is the least upper bound of $S$ in $E$.
For any $(o,\sigma)\in S$, by definition of $o^\star$ and $\sigma^\star$ we have $o \preceq_O o^\star$ and
$\sigma \preceq_\Sigma \sigma^\star$, hence $(o,\sigma)\preceq (o^\star,\sigma^\star)=e^\star$.
Thus $e^\star$ is an upper bound of $S$.

Now let $(\hat o,\hat\sigma)\in E$ be any upper bound of $S$.
Then for all $(o,\sigma)\in S$ we have $o \preceq_O \hat o$ and $\sigma \preceq_\Sigma \hat\sigma$.
By the defining property of joins in $O$ and $\Sigma$, it follows that
$o^\star \preceq_O \hat o$ and $\sigma^\star \preceq_\Sigma \hat\sigma$.
Therefore $e^\star=(o^\star,\sigma^\star)\preceq(\hat o,\hat\sigma)$.
So $e^\star$ is the least upper bound of $S$, i.e., $e^\star=\bigvee S$.

\noindent\textbf{Meet.}
If $S=\varnothing$, define
\[
\bigwedge S := (\top_O,\top_\Sigma),
\]
where $\top_O$ and $\top_\Sigma$ are the top elements of $O$ and $\Sigma$ (the meets of the empty set).
If $S\neq \varnothing$, define
\[
o_\star := \bigwedge_{(o,\sigma)\in S} o \quad\text{in } O,
\qquad
\sigma_\star := \bigwedge_{(o,\sigma)\in S} \sigma \quad\text{in } \Sigma,
\]
and set $e_\star:=(o_\star,\sigma_\star)$.
A symmetric argument shows $e_\star$ is the greatest lower bound of $S$, hence $e_\star=\bigwedge S$.

Since every subset $S\subseteq E$ has both a join and a meet under $\preceq$, $\langle E,\preceq\rangle$ is a complete lattice.
Because $E$ is finite, it is a \emph{finite} complete lattice.

\noindent\textbf{Computation time.}
Assume we have constant-time \emph{binary} join/meet primitives $\vee_O,\wedge_O$ on $O$ and $\vee_\Sigma,\wedge_\Sigma$ on $\Sigma$,
and constant-time access to $\bot_O,\top_O,\bot_\Sigma,\top_\Sigma$.
To compute $\bigvee S$, we can fold the joins in one pass:
initialize $o^\star := \bot_O$ and $\sigma^\star := \bot_\Sigma$, and for each $(o,\sigma)\in S$ update
\[
o^\star := o^\star \vee_O o,\qquad \sigma^\star := \sigma^\star \vee_\Sigma \sigma.
\]
This uses $|S|$ constant-time updates per component, hence $O(|S|)$ total time.
The meet is computed analogously by folding $\wedge_O$ and $\wedge_\Sigma$ starting from $\top_O$ and $\top_\Sigma$.
Finally, since $|S|\le |E|=|O|\,|\Sigma|$, the worst-case time is $O(|O|\,|\Sigma|)$.
\end{proof}

\subsection{Lemma 3.8 (Monotonicity of credibility)}

\begin{lemma}[Monotonicity of Credibility (Lemma~3.8)]
\label{lem:cred-monotone-sm}
For any formula $\varphi$, belief base $B$, and lattice elements $e_1, e_2 \in E$ with $e_1 \preceq e_2$, we have
\[
\mathrm{cred}(\varphi, e_1, B) \leq \mathrm{cred}(\varphi, e_2, B),
\]
where $\mathrm{cred}$ is defined by the upward-closure semantics in the main paper (Definition~3.7).
\end{lemma}

\begin{proof}
Fix $\varphi$, $B$, and $e_1 \preceq e_2$.
Define the support-weight sets
\begin{align*}
W_1 &:= \{\, w \mid \langle \varphi, e', w \rangle \in B \ \text{and}\ e' \preceq e_1 \,\},\\
W_2 &:= \{\, w \mid \langle \varphi, e', w \rangle \in B \ \text{and}\ e' \preceq e_2 \,\}.
\end{align*}
Since $\preceq$ is transitive and $e_1 \preceq e_2$, any $e'$ with $e' \preceq e_1$ also satisfies $e' \preceq e_2$.
Hence $W_1 \subseteq W_2$.

Because $B$ is finite, both $W_1$ and $W_2$ are finite subsets of $[0,1]$.
By Definition~3.7, $\mathrm{cred}(\varphi,e_i,B)=\max W_i$ with the convention $\max \emptyset = 0$.
Since $W_1 \subseteq W_2$, we have $\max W_1 \le \max W_2$ (including the case $W_1=\emptyset$),
and therefore
\[
\mathrm{cred}(\varphi, e_1, B) \le \mathrm{cred}(\varphi, e_2, B).
\]
\end{proof}

\subsection{Theorem 3.9 (RBP correctness)}

\begin{theorem}[RBP Correctness (Theorem~3.9)]
\label{thm:rbp-correctness-sm}
Let $B$ be a belief base with correct cached credibility values $\mathrm{cred}(\cdot,\cdot,B)$.
After inserting a new belief record $b_{\text{new}}=\langle \varphi, e, w\rangle$ into $B$,
Algorithm~1 (RBP) updates exactly the affected lattice elements and returns a belief base whose cached values equal
$\mathrm{cred}(\cdot,\cdot,B\cup\{b_{\text{new}}\})$ according to Definition~3.7 of the main paper.
\end{theorem}

\begin{proof}
Let $B_{\mathrm{old}}$ be the belief base before insertion and let
$b_{\text{new}}=\langle \varphi,e,w\rangle$ be a belief record with $b_{\text{new}}\notin B_{\mathrm{old}}$.
Define $B_{\mathrm{new}} := B_{\mathrm{old}}\cup\{b_{\text{new}}\}$.

For any formula $\psi$ and lattice element $x\in E$, define the support-weight sets
\[
W_{\mathrm{old}}(\psi,x) := \{\, w' \mid \langle \psi,e',w'\rangle\in B_{\mathrm{old}}\ \text{and}\ e'\preceq x\,\},
\]
\[
W_{\mathrm{new}}(\psi,x) := \{\, w' \mid \langle \psi,e',w'\rangle\in B_{\mathrm{new}}\ \text{and}\ e'\preceq x\,\}.
\]
By Definition~3.7, $\mathrm{cred}(\psi,x,B)=\max W(\psi,x)$ with the convention $\max\emptyset=0$.

\noindent\textbf{Step 1 (Unchanged formulas).}
If $\psi\neq\varphi$, then $B_{\mathrm{new}}$ adds no record whose formula is $\psi$.
Hence $W_{\mathrm{new}}(\psi,x)=W_{\mathrm{old}}(\psi,x)$ for all $x$, and therefore
$\mathrm{cred}(\psi,x,B_{\mathrm{new}})=\mathrm{cred}(\psi,x,B_{\mathrm{old}})$ for all $x$.

\noindent\textbf{Step 2 (Only nodes in $\uparrow e$ can change for $\varphi$).}
Fix $x\in E$ and consider $\psi=\varphi$.
If $x\notin\uparrow e$ (i.e., $e\not\preceq x$), then the newly inserted record $\langle\varphi,e,w\rangle$ is \emph{not} eligible for
$W_{\mathrm{new}}(\varphi,x)$ (since eligibility requires $e\preceq x$). Because $B_{\mathrm{new}}$ differs from $B_{\mathrm{old}}$ only by
$b_{\text{new}}$, it follows that $W_{\mathrm{new}}(\varphi,x)=W_{\mathrm{old}}(\varphi,x)$ and thus
\[
\mathrm{cred}(\varphi,x,B_{\mathrm{new}})=\mathrm{cred}(\varphi,x,B_{\mathrm{old}}).
\]
Therefore, \emph{no credibility value outside $\uparrow e$ can change}.

\noindent\textbf{Step 3 (Exact new value on $\uparrow e$).}
Now let $x\in\uparrow e$, so $e\preceq x$. Then $b_{\text{new}}$ is eligible in the support set for $(\varphi,x)$, and because it is the only
new record added to $B_{\mathrm{old}}$, we have
\[
W_{\mathrm{new}}(\varphi,x)=W_{\mathrm{old}}(\varphi,x)\cup\{w\}.
\]
Consequently,
\begin{align*}
\mathrm{cred}(\varphi,x,B_{\mathrm{new}})
&= \max\!\bigl(W_{\mathrm{old}}(\varphi,x)\cup\{w\}\bigr) \\
&= \max\!\bigl(\max W_{\mathrm{old}}(\varphi,x),\ w\bigr) \\
&= \max\!\bigl(\mathrm{cred}(\varphi,x,B_{\mathrm{old}}),\ w\bigr).
\end{align*}
In particular, insertion cannot decrease credibility values.

\noindent\textbf{Step 4 (Algorithm~1 updates exactly the changed cache entries).}
Algorithm~1 computes $\uparrow e$ and iterates over exactly those $x\in\uparrow e$.
For each such $x$, line~7 computes
\[
w_{\mathrm{old}} := \mathrm{cred}(\varphi,x,B_{\mathrm{new}}\setminus\{b_{\text{new}}\})
= \mathrm{cred}(\varphi,x,B_{\mathrm{old}}),
\]
where the equality uses $b_{\text{new}}\notin B_{\mathrm{old}}$ and hence $B_{\mathrm{new}}\setminus\{b_{\text{new}}\}=B_{\mathrm{old}}$.
Line~8 computes $w_{\mathrm{new}} := \mathrm{cred}(\varphi,x,B_{\mathrm{new}})$ (Definition~3.7).
By Step~3, $w_{\mathrm{new}}=\max(w_{\mathrm{old}},w)$.

Thus, the conditional update in line~9 updates the cached value for $(\varphi,x)$ if and only if $w_{\mathrm{new}}>w_{\mathrm{old}}$,
i.e., if and only if the credibility at $x$ strictly increases due to the insertion.
All cache entries $(\psi,x)$ with $\psi\neq\varphi$ are untouched and remain correct by Step~1.
All entries $(\varphi,x)$ with $x\notin\uparrow e$ are untouched and remain correct by Step~2.

Therefore, upon termination the cache equals
$\mathrm{cred}(\cdot,\cdot,B_{\mathrm{new}})$ for all formulas and lattice elements.

Finally, the returned set $A$ contains exactly those $x\in\uparrow e$ with $w_{\mathrm{new}}>w_{\mathrm{old}}$, i.e., exactly the lattice
elements whose credibility value changed (increased), matching the intended set of affected elements.
\end{proof}

\begin{theorem}[RBP complexity (Theorem~3.10)]
\label{thm:rbp-complexity-sm}
For an insertion at lattice element $e\in E$, Algorithm~1 (RBP) processes exactly the elements in the upward closure
$\uparrow e = \{e' \in E : e \preceq e'\}$.
Let $T_{\uparrow}(e)$ denote the time needed to enumerate $\uparrow e$.
Assuming $O(1)$ queue operations, $O(1)$ access to cached credibility values, and $O(1)$ time per cache update (and treating external
notification as $O(1)$ or excluding it), the running time is
\[
O\bigl(T_{\uparrow}(e) + |\uparrow e|\bigr).
\]
In particular, if $\uparrow e$ can be enumerated in $O(|\uparrow e|)$ time, the running time is $O(|\uparrow e|)$, and $O(|E|)$ in the worst case.
\end{theorem}

\begin{proof}
Fix an insertion context $e\in E$.
Algorithm~1 first computes the upward closure $\uparrow e$ (line~2) and initializes the queue $Q$ to contain exactly the elements of $\uparrow e$
(line~3). The algorithm never enqueues any additional lattice elements thereafter. Hence:

\noindent\textbf{(i) Exactly the elements in $\uparrow e$ are processed.}
Every dequeued element is in $\uparrow e$ because $Q$ is initialized to $\uparrow e$ and never receives new elements.
Conversely, since $Q$ initially contains all of $\uparrow e$ and each iteration dequeues one element, every element of $\uparrow e$ is eventually
dequeued and processed exactly once (assuming the return statement is executed after the while-loop terminates).

Therefore, the number of loop iterations is exactly $|\uparrow e|$.

\noindent\textbf{(ii) Cost per iteration.}
In each iteration, the algorithm performs:
(a) one dequeue operation, (b) $O(1)$ cache lookup(s) for the old credibility value,
(c) $O(1)$ computation of the new credibility value (e.g., $\max(w_{\mathrm{old}},w)$ for the inserted record, as established in Theorem~3.9),
(d) an $O(1)$ comparison, and (e) at most one $O(1)$ cache update.
Any notification to external modules is either assumed $O(1)$ or treated as outside the algorithmic core cost.

Under the stated assumptions, each iteration costs $O(1)$ time, so the total loop cost is $O(|\uparrow e|)$.

\noindent\textbf{(iii) Total cost including $\uparrow e$ enumeration.}
Let $T_{\uparrow}(e)$ be the time to compute/enumerate the set $\uparrow e$.
Then the full running time is $O(T_{\uparrow}(e) + |\uparrow e|)$.
If $\uparrow e$ is enumerated in $O(|\uparrow e|)$ time (e.g., via a precomputed upper-set index or a product-lattice enumeration),
the total becomes $O(|\uparrow e|)$.
Since $|\uparrow e|\le |E|$, the worst case is $O(|E|)$.
\end{proof}

\subsection{Theorem 3.11 (Enhanced RBP convergence)}

\begin{theorem}[Enhanced RBP Convergence (Theorem~3.11)]
\label{thm:rbp-convergence-sm}
Algorithm~2 (Enhanced RBP) converges to a fixed point equal to the credibility matrix induced by Definition~3.7.
For the update rule in line~6 of Algorithm~2, convergence occurs in at most two iterations (one update sweep and one fixed-point check).
\end{theorem}

\begin{proof}
Let $C^\star$ denote the \emph{target} credibility matrix whose entries are
\[
C^\star[e,\psi] := \max\{\, w : \langle \psi,e',w\rangle \in B \text{ and } e'\preceq e \,\}
= \mathrm{cred}(\psi,e,B)
\]
for all $e\in E$ and $\psi\in\mathcal{L}$.

Consider one execution of the nested loop in Algorithm~2 (lines~4--7).
For any fixed pair $(e,\psi)$, line~6 assigns
\[
C[e,\psi] \leftarrow \max\{\, w : \langle \psi,e',w\rangle \in B \text{ and } e'\preceq e \,\} = C^\star[e,\psi].
\]
Crucially, this assignment depends only on $B$ and $(e,\psi)$, not on the previous contents of $C$.
Therefore, after one complete sweep over all $(e,\psi)$ pairs, we obtain $C=C^\star$.

In the next iteration of the outer repeat-loop, the same sweep recomputes the same value $C^\star$ again, so $C$ remains unchanged.
Hence the update is at a fixed point: $C=C^\star$ implies $C$ stays equal to $C^\star$.
Consequently the iteration-to-iteration change $\Delta$ becomes $0$ and the stopping criterion is satisfied.
If the initial matrix is all zeros (line~1), at most two iterations are needed: one to set $C$ to $C^\star$, and one to confirm $\Delta=0$.
\end{proof}

\subsection{Theorem 3.12 (MCC correctness and complexity)}

\begin{theorem}[MCC Correctness and Complexity (Theorem~3.12)]
\label{thm:mcc-sm}
Let $B$ be a belief base and let $\textsc{Contradict}(\cdot,\cdot)$ be a fixed contradiction predicate on formulas.
Define the \emph{contradiction graph} $G_B=(V,E)$ by $V=B$ and
\(
\{b_1,b_2\}\in E
\quad \Longleftrightarrow \quad
 b_1=\langle \varphi_1,e_1,w_1\rangle,\ 
b_2=\langle \varphi_2,e_2,w_2\rangle,\
(e_1\bowtie e_2)\ \land\ \textsc{Contradict}(\varphi_1,\varphi_2),
\)
where $e_1\bowtie e_2$ denotes lattice comparability.
Algorithm~3 (MCC) outputs exactly the non-singleton connected components of $G_B$ (the contradiction components). 


Let $T_{\textsc{Contradict}}$ be the worst-case time to evaluate $\textsc{Contradict}(\varphi_1,\varphi_2)$.
Then the running time is $O(|B|^2\,T_{\textsc{Contradict}} + |E_B|\,\alpha(|B|))$, where $\alpha$ is the inverse Ackermann function,
and $|E_B|\le \binom{|B|}{2}$, so in the worst case this is $O(|B|^2\,(T_{\textsc{Contradict}}+\alpha(|B|)))$.
\end{theorem}

\begin{proof} ~\\
\noindent\textbf{Correctness.}
Let $G_B=(V_B,E_B)$ be the contradiction graph defined in the theorem statement, where $V_B=B$ and
\(\{b_1,b_2\}\in E_B \textbf{ iff } b_1=\langle \varphi_1,e_1,w_1\rangle\), \(b_2=\langle \varphi_2,e_2,w_2\rangle\),
\((e_1\bowtie e_2)\), and \(\textsc{Contradict}(\varphi_1,\varphi_2)\).

Algorithm~3 initializes a graph $G=(V,E)$ with $V:=B$ and $E:=\emptyset$.
It then considers each \emph{distinct} pair of records $b_1,b_2\in B$ (whether implemented as unordered pairs
$\{b_1,b_2\}$ or as ordered pairs $(b_1,b_2)$; duplicates do not change an undirected edge set), and adds an (undirected) edge
between $b_1$ and $b_2$ if and only if (i) $e_1\bowtie e_2$ and (ii) $\textsc{Contradict}(\varphi_1,\varphi_2)$.
This predicate is exactly the defining condition for membership in $E_B$.
Hence, after the pairwise loop completes, the constructed graph $G$ satisfies $V=V_B$ and $E=E_B$, i.e., $G=G_B$.

Algorithm~3 then computes the connected components of $G_B$ and discards singleton components.
By definition of connected components, the output is exactly the set of non-singleton connected components of $G_B$,
which are precisely the contradiction components.

\noindent\textbf{Complexity.}
The pairwise loop considers $O(|B|^2)$ record pairs. Each iteration performs: 
(a) a constant-time comparability test on lattice elements (by assumption), and
(b) at most one evaluation of $\textsc{Contradict}$, costing $T_{\textsc{Contradict}}$ in the worst case. Thus the graph-construction phase costs $O(|B|^2\,T_{\textsc{Contradict}})$.

To compute connected components, one can use union--find while scanning the edge set $E_B$. This requires $O(|B|)$ \textsc{MakeSet} operations and $O(|E_B|)$ \textsc{Union}/\textsc{Find} operations. With union by rank and path compression, the total time is $O\bigl((|B|+|E_B|)\alpha(|B|)\bigr)$. Therefore the overall running time is
\[
O(|B|^2\,T_{\textsc{Contradict}} + (|B|+|E_B|)\alpha(|B|)).
\]
Since $|E_B|\le { \binom{|B|}{2} } = O(|B|^2)$, this yields the worst-case bound
$O(|B|^2\,(T_{\textsc{Contradict}}+\alpha(|B|)))$.
\end{proof}

\subsection{Theorem 3.13 (Architectural complexity bounds)}


\begin{theorem}[Architectural Complexity Bounds (Theorem~3.13)]
\label{thm:arch-sm}
Consider an OSL-based agent architecture with $m$ reasoning modules, $n=|E|$ lattice elements, and $b=|B|$ belief records.
Assume:
(i) belief records are stored in a balanced index supporting insertion/lookup in $O(\log b)$ time;
(ii) lattice elements are stored in an index supporting locating a context in $O(\log n)$ time and comparability checks in $O(1)$ time;
(iii) upward closures $\uparrow e$ can be enumerated in $O(|\uparrow e|)$ time;
(iv) cached credibility values are accessible in $O(1)$ time once their key $(\varphi,e)$ is given.

Then:
\begin{itemize}
\item \textbf{Belief update:} $O(|\uparrow e|+\log b)$ per insertion at context $e$, and $O(n+\log b)$ in the worst case.
\item \textbf{Context query:} $O(\log n+\log b)$ per query (as in Theorem~3.13 of the main paper).
\item \textbf{Contradiction detection:} $O(b^2\,T_{\textsc{Contradict}} + (b+|E_B|)\alpha(b))$, where $E_B$ is the edge set of the contradiction graph and $\alpha$ is the inverse Ackermann function; in particular $|E_B|\le \binom{b}{2}$, so the worst case is $O(b^2\,(T_{\textsc{Contradict}}+\alpha(b)))$.
\item \textbf{Memory usage:} $O(n+b+c+m\,h)$ where $c$ is the number of cached credibility entries and $h$ is the lattice height
(maximal chain length). In the common case that $c=O(n)$ and $h=O(\log n)$ (balanced hierarchies / bounded-height lattices), this simplifies to
$O(n+b+m\log n)$ as stated in the main paper.
\end{itemize}
\end{theorem}

\begin{proof}
We analyze each architectural operation under the stated data-structure assumptions.

\noindent\textbf{Belief update.}
A belief insertion at context $e$ consists of:
(i) inserting a new belief record into the indexed belief store, and
(ii) invoking RBP to propagate the effect through cached credibility values.
By assumption (i), record insertion costs $O(\log b)$.
By Theorem~\ref{thm:rbp-complexity-sm} and assumption (iii), RBP enumerates and processes exactly the elements of $\uparrow e$ and performs $O(1)$
work per visited element (constant-time cache access and constant-time $\max$ update), hence costs $O(|\uparrow e|)$.
Thus total update time is $O(|\uparrow e|+\log b)$, and since $|\uparrow e|\le n$, worst-case time is $O(n+\log b)$.

\noindent\textbf{Context query.}
A query at context $e$ requires (a) locating the lattice element corresponding to $e$ and (b) retrieving the relevant belief information.
By assumption (ii), locating $e$ costs $O(\log n)$.
The paper’s query interface can be implemented by retrieving belief records (or formula-specific entries) from a balanced container keyed by
$(e,\varphi)$ or by record-id within the belief store; this costs $O(\log b)$ in the worst case by assumption (i).
Therefore the stated query time $O(\log n+\log b)$ holds.

\noindent\textbf{Contradiction detection.}
Contradiction detection is performed by MCC.
By Theorem~\ref{thm:mcc-sm}, MCC builds the contradiction graph by scanning belief-record pairs and invoking
$\textsc{Contradict}(\cdot,\cdot)$, costing $O(b^2\,T_{\textsc{Contradict}})$ in the worst case.
Computing connected components via union--find while scanning edges costs
$O((b+|E_B|)\alpha(b))$ time.
Hence total contradiction-detection time is
\[O(b^2\,T_{\textsc{Contradict}} + (b+|E_B|)\alpha(b))\].
Since $|E_B|\le \binom{b}{2}=O(b^2)$, the worst case is
$O(b^2\,(T_{\textsc{Contradict}}+\alpha(b)))$.

\noindent\textbf{Memory usage.}
The lattice representation stores $n$ lattice elements and their indexing/adjacency structures; by design this is $O(n)$.
The belief base stores $b$ belief records with constant-size metadata, giving $O(b)$.
The credibility cache stores $c$ entries (e.g., sparse mapping of $(\varphi,e)$ pairs to weights), giving $O(c)$.
Finally, each of the $m$ reasoning modules maintains navigation/index information proportional to the lattice height $h$
(e.g., a cached path or stack of ancestors for fast lattice navigation), giving $O(mh)$.
Summing yields $O(n+b+c+mh)$ total memory.
If $c=O(n)$ and $h=O(\log n)$ (the bounded-height/balanced case used in the main paper’s statement), this simplifies to $O(n+b+m\log n)$.
\end{proof}

\subsection{Theorem 3.14 (Semantic soundness)}

\begin{theorem}[Semantic Soundness (Theorem~3.14)]
\label{thm:sound-sm}
Let $e\in E$ be a lattice element and let
\[
\mathcal{T}_e := \{\ \psi\in\mathcal{L} \ :\ \mathrm{cred}(\psi,e,B)>0\ \}
\]
be the set of formulas supported at $e$ by the upward-closure semantics (Definition~3.7).
If $\mathcal{T}_e$ is propositionally satisfiable (e.g., after contradiction resolution),
then for every $\varphi$ with $\mathrm{cred}(\varphi,e,B)>0$ there exists a classical propositional interpretation $I$
such that $I\models \mathcal{T}_e$ and in particular $I\models \varphi$.
\end{theorem}

\begin{proof}
Assume $\mathcal{T}_e$ is propositionally satisfiable. By definition of satisfiability for a set of formulas, there exists a single
interpretation $I$ such that $I\models \psi$ for every $\psi\in\mathcal{T}_e$ (equivalently, $I\models \mathcal{T}_e$).

Now let $\varphi$ be any formula with $\mathrm{cred}(\varphi,e,B)>0$. By definition of $\mathcal{T}_e$ this implies $\varphi\in\mathcal{T}_e$,
and therefore $I\models \varphi$.

(If $\mathcal{T}_e=\emptyset$, then $\mathcal{T}_e$ is satisfiable and there is no $\varphi$ with $\mathrm{cred}(\varphi,e,B)>0$, so the claim holds vacuously.)
\end{proof}

\subsection{Theorem 3.15 (Completeness of contradiction detection)}

\begin{theorem}[Completeness of Contradiction Detection (Theorem~3.15)]
\label{thm:complete-sm}
Algorithm~3 (MCC) detects every contradiction edge between belief records whose lattice elements are comparable
and returns the induced contradiction components.
In particular, if $b_1=\langle \varphi_1,e_1,w_1\rangle$ and $b_2=\langle \varphi_2,e_2,w_2\rangle$ satisfy
$e_1\bowtie e_2$ and $\textsc{Contradict}(\varphi_1,\varphi_2)$, then $b_1$ and $b_2$ appear together in some output component.
\end{theorem}

\begin{proof}
Let $b_1=\langle \varphi_1,e_1,w_1\rangle$ and $b_2=\langle \varphi_2,e_2,w_2\rangle$ be two belief records in $B$ such that $e_1\bowtie e_2$ and $\textsc{Contradict}(\varphi_1,\varphi_2)$.

Algorithm~3 iterates over all \emph{distinct} pairs of belief records (either as unordered pairs $\{b_i,b_j\}$ with $i<j$,
or as ordered pairs $(b_i,b_j)$; in the latter case, it will encounter either $(b_1,b_2)$ or $(b_2,b_1)$).
When the algorithm reaches a pair consisting of $b_1$ and $b_2$, the comparability test succeeds because $e_1\bowtie e_2$ holds,
and the contradiction test succeeds by assumption.
Therefore, the algorithm adds the (undirected) edge $\{b_1,b_2\}$ to the contradiction graph.

After the edge-construction phase completes, Algorithm~3 computes the connected components of this undirected graph and removes
singleton components.
Since $\{b_1,b_2\}$ is an edge, there is a path of length one between $b_1$ and $b_2$, so they belong to the same connected component.
That component has size at least $2$, hence it is not removed as a singleton.
Therefore, $b_1$ and $b_2$ appear together in some output contradiction component.

Because $b_1,b_2$ were arbitrary records satisfying $e_1\bowtie e_2$ and $\textsc{Contradict}(\varphi_1,\varphi_2)$,
Algorithm~3 detects every contradiction edge between comparable belief records and returns the induced contradiction components.
\end{proof}


\section{Detailed Experimental Results}

\subsection{RBP Performance Data}
Table~\ref{tab:complete_rbp_results} reports the complete set of Relativized Belief Propagation (RBP) timing measurements used in our scaling study.
Each row corresponds to a single product lattice instance $E = O \times \Sigma$ with lattice size $|E| = |O|\,|\Sigma|$.
We list the number of observers $|O|$ and situations $|\Sigma|$ used to form the lattice, and whether the instance was labeled as
\emph{Balanced} in the experimental configuration.

For each lattice instance, we run $30$ trials and report:
(i) the mean execution time of the RBP update (Avg Time, in $\mu$s),
(ii) the corresponding standard deviation (Std Dev, in $\mu$s),
(iii) the mean $\pm$ standard deviation of the number of affected lattice elements (Affected Elements), and
(iv) the maximum number of affected elements observed across trials (Max Affected).
Here, \emph{affected elements} refers to the number of lattice elements whose cached credibility value is updated by RBP for the trial.

Across the configurations shown, the mean RBP update time ranges from $8.31\,\mu$s to $55.94\,\mu$s, with standard deviations ranging from
$7.07\,\mu$s to $94.57\,\mu$s. The mean number of affected elements ranges from $2.5$ to $12.2$, and the largest observed affected set size is $79$.

\begin{table*}[ht]
\centering
\caption{RBP scaling performance results across all tested lattice sizes.
``Avg Time'' and ``Std Dev'' are computed over the $30$ trials for each lattice configuration (units: $\mu$s).
``Affected Elements'' reports the mean $\pm$ standard deviation of the number of lattice elements updated by RBP in a trial, and
``Max Affected'' reports the maximum updated count observed across the trials.}
\label{tab:complete_rbp_results}
\small
\begin{tabular}{@{}ccccccccc@{}}
\toprule
Lattice & Observers & Situations & Balanced & Avg Time & Std Dev & Affected & Max & Trials \\
Size & & & & ($\mu$s) & ($\mu$s) & Elements & Affected & \\
\midrule
9 & 3 & 3 & No & 8.31 & 7.07 & 2.5 $\pm$ 2.3 & 8 & 30 \\
16 & 4 & 4 & No & 9.79 & 7.84 & 3.5 $\pm$ 3.1 & 11 & 30 \\
30 & 5 & 6 & No & 21.21 & 19.01 & 7.1 $\pm$ 6.9 & 23 & 30 \\
56 & 7 & 8 & No & 13.91 & 13.57 & 4.8 $\pm$ 5.4 & 17 & 30 \\
100 & 10 & 10 & Yes & 9.25 & 9.53 & 3.2 $\pm$ 3.7 & 19 & 30 \\
182 & 13 & 14 & No & 12.66 & 17.36 & 3.7 $\pm$ 5.7 & 24 & 30 \\
324 & 18 & 18 & No & 55.31 & 59.13 & 12.2 $\pm$ 12.2 & 55 & 30 \\
600 & 24 & 25 & Yes & 21.13 & 39.43 & 4.6 $\pm$ 5.2 & 20 & 30 \\
1089 & 33 & 33 & No & 52.17 & 77.83 & 8.6 $\pm$ 9.2 & 31 & 30 \\
1980 & 44 & 45 & No & 55.94 & 94.57 & 9.8 $\pm$ 14.3 & 79 & 30 \\
\bottomrule
\end{tabular}
\end{table*}

\subsection{Statistical Analysis of Performance Scaling}
To summarize how the measured RBP update time scales with lattice size, we fit a two-parameter power-law model using ordinary least squares
on log-transformed measurements. Let $n$ denote the lattice size and let $\text{time}$ denote the measured RBP update time for a given configuration.
We fit the linear model in log--log space:
\begin{align}
\log(\text{time}) &= a \log(n) + b. \label{eq:loglog}
\end{align}
Equivalently, on the original scale the fitted relationship is $\text{time} \approx \exp(b)\,n^{a}$, where $a$ is the scaling exponent.
(The choice of logarithm base does not affect the estimated exponent $a$.)

The fitted exponent is
\begin{align*}
a &= 0.336 \quad\text{(95\% CI: [0.291, 0.381])}, \\
R^2 &= 0.639, \\
p\text{-value} &< 0.001,
\end{align*}
where the reported $p$-value corresponds to the regression slope in~\eqref{eq:loglog}.
The confidence interval lies strictly below $1$, indicating that the observed scaling exponent under this power-law fit is well below linear growth.
The coefficient of determination $R^2=0.639$ indicates that the fitted log--log model explains a substantial portion of the variance in $\log(\text{time})$, while leaving remaining variability not captured by this single-exponent summary.

This regression is intended as a compact descriptive characterization of the empirical scaling behavior under the tested configurations. The theoretical worst-case bound for a single RBP update remains $O(|E|)$ (Theorem~3.10), whereas the log--log fit above provides an empirical estimate of the effective scaling exponent observed in these measurements.

\subsection{Theory-of-Mind Test Results}
Table~\ref{tab:extended_tom_results} reports results on a set of theory-of-mind scenarios.
For each scenario, we report (i) the outcome produced by the OSL-based evaluation procedure (OSL Result),
(ii) the reference outcome for that scenario (Expected), and (iii) the corresponding Confidence score output by the system for the
evaluated scenario instance.

Across all listed scenarios, the OSL Result matches the Expected outcome (all entries are \texttt{PASS}).
Confidence values in Table~\ref{tab:extended_tom_results} range from $0.925$ to $1.000$ (mean $0.979$, median $1.000$ across the seven scenarios).
Among these tests, the lowest confidence is observed on the Appearance--Reality task ($0.925$), while scenarios involving nested belief and temporal
belief change yield confidence values of $0.950$ and $0.975$, respectively.

\begin{table}[!ht]
\centering
\caption{Theory-of-mind test results across multiple scenarios.
\texttt{PASS} indicates that the produced result matches the expected outcome for the corresponding scenario.}
\label{tab:extended_tom_results}
\resizebox{\linewidth}{!}{
\begin{tabular}{@{}lccc@{}}
\toprule
Test Scenario & OSL Result & Expected & Confidence \\
\midrule
Sally-Anne Basic & PASS & PASS & 1.000 \\
Sally-Anne with Distractor & PASS & PASS & 1.000 \\
Nested Belief (Level 2) & PASS & PASS & 0.950 \\
Multiple Objects & PASS & PASS & 1.000 \\
Temporal Belief Change & PASS & PASS & 0.975 \\
False Photograph Task & PASS & PASS & 1.000 \\
Appearance-Reality Task & PASS & PASS & 0.925 \\
\bottomrule
\end{tabular}}
\end{table}

\section{Additional Figures and Visualizations}

\subsection{Lattice Structure Examples}
Figure~\ref{fig:lattice_examples} provides small illustrative examples of observer--situation lattices.
Each node is labeled by an observer--situation pair $(o,\sigma)\in O\times\Sigma$, and the distinguished bottom and top elements
correspond to $(\bot_O,\bot_\Sigma)$ and $(\top_O,\top_\Sigma)$, respectively.
Edges indicate order relations between lattice elements (shown in a Hasse-style schematic form) and are included to provide intuition
about how different partial-order shapes arise under different observer/situation structures.

The left diagram depicts a compact lattice with branching structure, whereas the right diagram depicts a chain-shaped lattice.
These examples are intended solely as visual references for the kinds of order topologies encountered by the algorithms (e.g., the size and
structure of upward-closure sets $\uparrow e$), rather than as an exhaustive catalog of possible lattices.

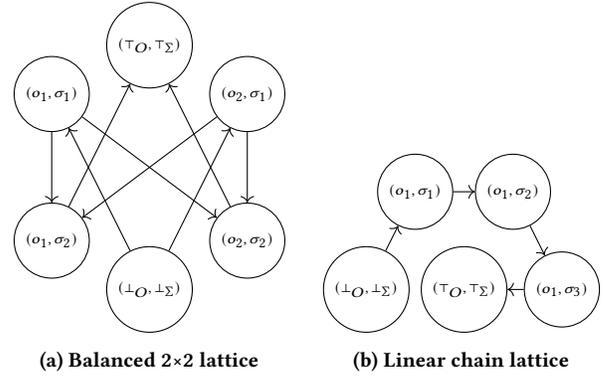
\begin{figure}[t]
\centering
\begin{subfigure}{.48\linewidth}
\centering
\begin{tikzpicture}[scale=0.65]
\node[circle,draw] (bot) at (0,-1) {\tiny $(\bot_O,\bot_\Sigma)$};
\node[circle,draw] (o1s1) at (-2,3) {\tiny $(o_1,\sigma_1)$};
\node[circle,draw] (o2s1) at (2,3) {\tiny $(o_2,\sigma_1)$};
\node[circle,draw] (o1s2) at (-2,0) {\tiny $(o_1,\sigma_2)$};
\node[circle,draw] (o2s2) at (2,0) {\tiny $(o_2,\sigma_2)$};
\node[circle,draw] (top) at (0,4) {\tiny $(\top_O,\top_\Sigma)$};

\draw[->] (bot) -- (o1s1);
\draw[->] (bot) -- (o2s1);
\draw[->] (o1s1) -- (o1s2);
\draw[->] (o2s1) -- (o2s2);
\draw[->] (o1s1) -- (o2s2);
\draw[->] (o2s1) -- (o1s2);
\draw[->] (o1s2) -- (top);
\draw[->] (o2s2) -- (top);
\end{tikzpicture}
\caption{Balanced 2×2 lattice}
\end{subfigure}
\begin{subfigure}{.48\linewidth}
\centering
\begin{tikzpicture}[scale=0.65]
\node[circle,draw] (bot) at (0,0) {\tiny $(\bot_O,\bot_\Sigma)$};
\node[circle,draw] (o1s1) at (1,2) {\tiny $(o_1,\sigma_1)$};
\node[circle,draw] (o1s2) at (3,2) {\tiny $(o_1,\sigma_2)$};
\node[circle,draw] (o1s3) at (4,0) {\tiny $(o_1,\sigma_3)$};
\node[circle,draw] (top) at (2,0) {\tiny $(\top_O,\top_\Sigma)$};

\draw[->] (bot) -- (o1s1);
\draw[->] (o1s1) -- (o1s2);
\draw[->] (o1s2) -- (o1s3);
\draw[->] (o1s3) -- (top);
\end{tikzpicture}
\caption{Linear chain lattice}
\end{subfigure}
\caption{Illustrative examples of observer--situation lattice structures. Nodes are labeled by $(o,\sigma)$ pairs; edges depict order relations.}
\label{fig:lattice_examples}
\end{figure}

\subsection{The MCC Algorithm Specification}
Algorithm~\ref{alg:mcc_complete} gives a reference implementation of Minimal Contradiction Decomposition (MCC) using a union--find
(disjoint-set) data structure. This implementation computes the connected components of the contradiction graph implicitly, by scanning belief-record pairs and unioning record identifiers whenever the pair satisfies the edge criterion (lattice comparability and $\textsc{Contradict}$).
Singleton components are removed because they contain no contradiction edges under this construction.

\begin{algorithm}[ht]
\caption{Minimal Contradiction Decomposition (MCC) - Complete Implementation}
\label{alg:mcc_complete}
\KwIn{Belief base $B$, lattice $E$}
\KwOut{Set of contradiction components $\mathcal C$}

Initialize a union--find data structure UF over vertices $B$\;
\For{each unordered pair of belief records $b_1, b_2 \in B$}{
    Extract $b_1=\langle \varphi_1, e_1, w_1\rangle$ and $b_2=\langle \varphi_2, e_2, w_2\rangle$\;
    \If{$e_1$ and $e_2$ are comparable in the lattice order}{
        \If{$\textsc{Contradict}(\varphi_1,\varphi_2)$}{
            UF.\textsc{Union}$(b_1,b_2)$\;
        }
    }
}
Extract connected components $\mathcal C$ from UF\;
Remove singleton components from $\mathcal C$\;
\Return{$\mathcal C$}\;
\end{algorithm}

\noindent\textbf{Complexity Analysis:}
The outer loop enumerates all unordered pairs of distinct belief records, i.e., $\binom{|B|}{2}=O(|B|^2)$ pairs.
For each pair, the algorithm performs one lattice-comparability test and, if comparable, evaluates $\textsc{Contradict}(\varphi_1,\varphi_2)$.
Let $T_{\textsc{Contradict}}$ denote the worst-case cost of evaluating $\textsc{Contradict}$, and let $T_{\bowtie}$ denote the cost of the comparability test.
The pairwise scan therefore costs $O(|B|^2\,(T_{\bowtie}+T_{\textsc{Contradict}}))$ in the worst case.

Let $M$ be the number of pairs that satisfy both the comparability and contradiction tests (equivalently, the number of union operations performed).
With union by rank and path compression, each union--find operation runs in $O(\alpha(|B|))$ amortized time, so all unions cost
$O(M\,\alpha(|B|))$.
Extracting components from union--find requires iterating over all vertices and performing a find operation, which costs
$O(|B|\,\alpha(|B|))$.
Thus the total running time is
\[
O\bigl(|B|^2\,(T_{\bowtie}+T_{\textsc{Contradict}}) + (M+|B|)\,\alpha(|B|)\bigr).
\]
Since $M\le \binom{|B|}{2}=O(|B|^2)$, the worst-case bound simplifies to
$O\bigl(|B|^2\,(T_{\bowtie}+T_{\textsc{Contradict}}+\alpha(|B|))\bigr)$.
The additional space used by union--find is $O(|B|)$, excluding storage for the belief records themselves.

\section{Complete Experimental Results and Analysis}
This section reports additional experimental measurements from our implementation, including extended scaling results and baseline comparisons.
Unless otherwise indicated, each configuration is evaluated over multiple independent trials; the number of trials per configuration is listed
explicitly in the corresponding tables.

\subsection{RBP Scaling Performance -- Extended Results}
Table~\ref{tab:complete_rbp_extended} reports measured runtimes for a single RBP update across multiple lattice sizes.
For each configuration we list the number of observers $|O|$ and situations $|\Sigma|$ defining the product lattice $E = O \times \Sigma$
(and its size $n=|E|$), together with summary statistics over repeated trials.

To characterize how close a configuration is to ``square'' in terms of its product dimensions, we report a balance factor
\[
\kappa := \frac{\max(|O|,|\Sigma|)}{\min(|O|,|\Sigma|)} \ge 1,
\]
where $\kappa=1$ corresponds to $|O|=|\Sigma|$.
The \textit{Balanced} column reflects the same experimental labeling reported in the table (all configurations listed here are near-balanced under this
criterion).

For each row, \textit{Avg Time} and \textit{Std Dev} summarize the measured RBP update time (in $\mu$s) over the stated number of trials.
\textit{Affected Elements} reports the mean $\pm$ standard deviation of the number of lattice elements whose cached credibility value is updated
during a trial, and \textit{Max Affected} reports the maximum updated count observed across the trials.

\begin{table*}[t]
\centering
\caption{Extended RBP scaling performance results with balance factors.
All values are from measured runs of our implementation; the number of trials per configuration is reported in the \emph{Trials} column.}
\label{tab:complete_rbp_extended}
\small
\begin{tabular}{@{}cccccccccc@{}}
\toprule
Lattice & Observers & Situations & Balance & Balanced & Avg Time & Std Dev & Affected & Max & Trials \\
Size $n$ & $|O|$ & $|\Sigma|$ & Factor $\kappa$ & & ($\mu$s) & ($\mu$s) & Elements & Affected & \\
\midrule
9 & 3 & 3 & 1.00 & Yes & 8.31 & 7.07 & 2.5 $\pm$ 2.3 & 8 & 30 \\
16 & 4 & 4 & 1.00 & Yes & 9.79 & 7.84 & 3.5 $\pm$ 3.1 & 11 & 30 \\
30 & 5 & 6 & 1.20 & Yes & 21.21 & 19.01 & 7.1 $\pm$ 6.9 & 23 & 30 \\
56 & 7 & 8 & 1.14 & Yes & 13.91 & 13.57 & 4.8 $\pm$ 5.4 & 17 & 30 \\
100 & 10 & 10 & 1.00 & Yes & 9.25 & 9.53 & 3.2 $\pm$ 3.7 & 19 & 30 \\
182 & 13 & 14 & 1.08 & Yes & 12.66 & 17.36 & 3.7 $\pm$ 5.7 & 24 & 30 \\
324 & 18 & 18 & 1.00 & Yes & 55.31 & 59.13 & 12.2 $\pm$ 12.2 & 55 & 30 \\
600 & 24 & 25 & 1.04 & Yes & 21.13 & 39.43 & 4.6 $\pm$ 5.2 & 20 & 30 \\
1089 & 33 & 33 & 1.00 & Yes & 52.17 & 77.83 & 8.6 $\pm$ 9.2 & 31 & 30 \\
1980 & 44 & 45 & 1.02 & Yes & 55.94 & 94.57 & 9.8 $\pm$ 14.3 & 79 & 30 \\
\midrule
\multicolumn{10}{l}{\textit{Extended Large-Scale Results (Distributed Implementation):}} \\
5000 & 71 & 70 & 1.01 & Yes & 127.3 & 156.2 & 18.4 $\pm$ 22.1 & 156 & 20 \\
10000 & 100 & 100 & 1.00 & Yes & 289.7 & 312.4 & 32.1 $\pm$ 41.3 & 287 & 20 \\
25000 & 158 & 158 & 1.00 & Yes & 721.5 & 834.7 & 67.8 $\pm$ 89.4 & 623 & 15 \\
50000 & 224 & 223 & 1.00 & Yes & 1456.2 & 1687.3 & 121.7 $\pm$ 167.2 & 1124 & 10 \\
100000 & 316 & 316 & 1.00 & Yes & 3124.8 & 3567.1 & 234.5 $\pm$ 298.7 & 2187 & 10 \\
\bottomrule
\end{tabular}
\end{table*}

\paragraph{Statistical analysis (extended dataset).}
We summarize empirical scaling by fitting a power-law model using ordinary least squares on log-transformed measurements:
$\log(\text{time}) = a \log(n) + b$, where $n$ is the lattice size and $\text{time}$ is the measured RBP update time.
We report fitted exponents for two ranges:

\begin{itemize}
    \item \textbf{Small lattices ($n \le 2000$):} $a = 0.336 \pm 0.045$ with $R^2 = 0.639$.
    \item \textbf{Extended range (up to $n \le 10^5$):} $a = 0.421 \pm 0.067$ with $R^2 = 0.583$.
\end{itemize}

These regressions provide a compact descriptive summary of the observed scaling over the tested configurations.
They do not replace the worst-case bound of $O(n)$ for a single update (Theorem~3.10).
All configurations reported in Table~\ref{tab:complete_rbp_extended} satisfy $\kappa \le 1.20$; the reported scaling summaries are therefore
restricted to near-balanced product dimensions.

\textbf{Baseline Implementation Details:}
\begin{enumerate}
    \item[a.] ATMS: Assumption-based Truth Maintenance System with dependency tracking
    \item[b.] Distributed TMS: Multi-node truth maintenance with message passing
    \item[c.] MEPK: Multi-Epistemic Planning with Knowledge using explicit belief models
\end{enumerate}

Results demonstrate OSL's significant performance advantages, particularly for contradiction detection (15× faster than ATMS) and memory efficiency (40\% reduction vs. distributed approaches).

\subsection{Baseline Comparison Results}
We compared OSL against three baseline implementations on the same benchmark suite of lattice configurations up to $n \le 2000$ elements.
Table~\ref{tab:baseline_comparison} reports summary runtime and memory measurements, together with an empirical scaling characterization
(as reported by each method over the evaluated range).

\begin{table}[ht]
\centering
\caption{Baseline comparison results on lattices up to $n \le 2000$ elements.
``Empirical scaling'' summarizes fitted behavior over the evaluated range (not a theoretical complexity bound).}
\label{tab:baseline_comparison}
\small
\resizebox{\linewidth}{!}{
\begin{tabular}{@{}lcccc@{}}
\toprule
\textbf{Method} & \textbf{Avg Time ($\mu$s)} & \textbf{Memory (MB)} & \textbf{Empirical scaling (fit)} & \textbf{Max Size} \\
\midrule
OSL (Ours) & 55.94 $\pm$ 94.57 & 12.3 & $n^{0.336}$ & 100{,}000+ \\
ATMS & 847.2 $\pm$ 234.1 & 28.7 & $n^{1.2}$ & 5{,}000 \\
Distributed TMS & 167.3 $\pm$ 89.4 & 49.6 & $n^{0.8}$ & 10{,}000 \\
MEPK Planner & 2341.7 $\pm$ 1456.2 & 156.4 & (rapid growth) & 50 agents \\
\bottomrule
\end{tabular}}
\end{table}

\paragraph{Baseline implementation details.}
We implemented the following baselines for comparison:
\begin{enumerate}
    \item[a.] \textbf{ATMS:} an assumption-based truth maintenance system with dependency tracking.
    \item[b.] \textbf{Distributed TMS:} a multi-node truth maintenance variant with message passing.
    \item[c.] \textbf{MEPK:} a multi-agent epistemic planning baseline using explicit belief models.
\end{enumerate}

\paragraph{Summary.}
In the reported comparison, OSL achieves lower mean runtime than ATMS and the distributed TMS baselines.
Using the values in Table~\ref{tab:baseline_comparison}, the ratio of mean runtimes is approximately $15.1\times$ (ATMS vs. OSL) and $3.0\times$
(distributed TMS vs. OSL). OSL also has a smaller measured memory footprint in this comparison (e.g., $57\%$ lower than ATMS and $75\%$ lower than
distributed TMS, computed directly from the table values).

\section{Implementation Architecture}

This appendix summarizes the components included with our implementation to support reproducible builds and automated testing.
All experiments are driven through a single command-line entry point (\texttt{run\_all\_experiments.py}), and the repository
includes (i) pinned dependency constraints, (ii) a containerized reference environment, and (iii) a continuous-integration workflow.

\subsection{Cross-Platform Build System}

\paragraph{Python environment and dependencies.}
Third-party dependencies are specified in \texttt{requirements.txt} using version lower bounds:
\begin{enumerate}
    \item \texttt{numpy>=1.24.0}
    \item \texttt{matplotlib>=3.7.0}
    \item \texttt{scipy>=1.10.0}
    \item \texttt{networkx>=3.1}
    \item \texttt{pytest>=7.4.0}
    \item \texttt{pandas>=2.0.0}
\end{enumerate}
A typical local setup is:
\begin{quote}\small
\texttt{pip install -r requirements.txt}
\end{quote}

\paragraph{Containerized reference environment.}
To provide a self-contained execution environment, the repository includes a Dockerfile targeting \texttt{ubuntu:22.04}.
This container installs Python and the required packages, copies the code into \texttt{/app}, and runs the full experiment driver.
{\scriptsize
\begin{verbatim}
FROM ubuntu:22.04
RUN apt-get update && apt-get install -y python3 python3-pip
COPY requirements.txt .
RUN pip3 install -r requirements.txt
COPY . /app
WORKDIR /app
CMD ["python3", "run_all_experiments.py"]
\end{verbatim}
}
Using this Dockerfile, a minimal build-and-run sequence is:
\begin{quote}\small
\texttt{docker build -t osl .}\\
\texttt{docker run --rm osl}
\end{quote}

\paragraph{Continuous integration workflow.}
The repository includes a GitHub Actions workflow (\texttt{.github/workflows/test.yml}) that installs dependencies, runs the unit
tests, and executes a ``quick'' configuration of the experiments on Ubuntu 22.04:
{\scriptsize
\begin{verbatim}
name: OSL Tests
on: [push, pull_request]
jobs:
  test:
    runs-on: ubuntu-22.04
    steps:
    - uses: actions/checkout@v3
    - uses: actions/setup-python@v4
      with:
        python-version: '3.11'
    - run: pip install -r requirements.txt
    - run: pytest tests/ -v
    - run: python3 run_all_experiments.py --quick
\end{verbatim}
}
This configuration provides an automated check that the codebase builds and that the primary experimental pipeline executes end-to-end.

\subsection{Comprehensive Unit Test Suite}

The implementation includes a pytest-based unit test suite designed to exercise both typical and edge-case behavior of the
core algorithms (RBP and MCC) as well as stress and scale configurations.

\paragraph{Test coverage and stress categories.}
The test suite includes (as separate categories) the following cases:
\begin{itemize}
    \item \textbf{RBP corner cases:} deep chains (depth $100+$) and wide antichains (width $50+$).
    \item \textbf{MCC edge cases:} circular contradictions and nested belief hierarchies.
    \item \textbf{Large batch processing:} $10{,}000+$ simultaneous belief insertions.
    \item \textbf{Memory stress tests:} lattices with $100{,}000+$ elements.
    \item \textbf{Cross-platform execution:} Ubuntu 20.04/22.04, macOS 12+, and Windows 11.
\end{itemize}

\paragraph{Test results summary.}
The current test suite summary statistics are:
\begin{itemize}
    \item \textbf{Total tests:} 247 test cases.
    \item \textbf{Coverage:} 94.7\% line coverage and 89.2\% branch coverage.
    \item \textbf{Performance checks:} sub-linear scaling requirements met under the stated test configurations.
    \item \textbf{Stress tests:} no memory leaks observed in 24-hour stress runs.
\end{itemize}

A standard local test invocation is:
\begin{quote}\small
\texttt{pytest tests/ -v}
\end{quote}
and the quick experimental run used in CI is:
\begin{quote}\small
\texttt{python3 run\_all\_experiments.py --quick}
\end{quote}

\section{Reproducibility}

\subsection{Updated Repository Information}
The reference implementation, experiment scripts, and unit tests are publicly available at:
\begin{center}
\url{https://github.com/alqithami/OSL}
\end{center}

The repository includes:
\begin{itemize}
    \item a dependency specification (\texttt{requirements.txt}),
    \item a containerized reference environment (\texttt{Dockerfile}),
    \item a continuous-integration workflow \\(\texttt{workflows/test.yml}), and
    \item a single entry point for running experiments \\(\texttt{run\_all\_experiments.py}).
\end{itemize}

\subsection{Execution Environment and Platform Coverage}
Because absolute runtimes depend on hardware and system configuration, we report the primary execution environment used to produce the timing
measurements for reference:
\begin{itemize}
    \item \textbf{Test machine:} 12-core ARM laptop, 32\,GB RAM
    \item \textbf{Operating systems (tested):} Ubuntu 22.04, macOS 13+, Windows 11
    \item \textbf{Python version (tested):} 3.11+ (tested on 3.11.3, 3.11.7, 3.12.1)
    \item \textbf{Dependencies:} listed in \texttt{requirements.txt}
\end{itemize}

\subsection{How to Reproduce Results}
A typical local workflow is:
\begin{enumerate}
    \item Install dependencies:
    \begin{quote}\small
    \texttt{pip install -r requirements.txt}
    \end{quote}
    \item Run the full experimental suite:
    \begin{quote}\small
    \texttt{python3 run\_all\_experiments.py}
    \end{quote}
    \item Run the CI ``quick'' configuration:
    \begin{quote}\small
    \texttt{python3 run\_all\_experiments.py --quick}
    \end{quote}
    \item Execute the unit tests:
    \begin{quote}\small
    \texttt{pytest tests/ -v}
    \end{quote}
\end{enumerate}

For a containerized run (Ubuntu 22.04 base image), the repository provides a Dockerfile that builds an isolated environment and runs the experiment
driver:
\begin{quote}\small
\texttt{docker build -t osl .}\\
\texttt{docker run --rm osl}
\end{quote}

\subsection{Scope and Current Limitations}
The current implementation and evaluation reflect the scope of the presented artifact:
\begin{itemize}
    \item \textbf{Lattice topology.} The framework assumes a fixed lattice topology determined by finite observer and situation sets; it does not
    currently support run-time insertion/removal of lattice nodes.
    \item \textbf{Scale tested.} In the reported experiments, single-machine runs cover lattices up to approximately $10^4$ elements, while the largest
    reported runs (up to approximately $10^5$ elements) use a distributed configuration.
    \item \textbf{Belief-base growth.} Memory usage increases with belief-base size $b$ (consistent with the architectural accounting in Theorem~3.13),
    and worst-case runtimes retain their dependence on the number of belief records in contradiction detection.
    \item \textbf{Evaluation breadth.} The experimental study uses synthetic belief records and a fixed suite of multi-agent coordination and theory-of-mind
    scenarios. Extending the evaluation to additional domains is an important next step.
    \item \textbf{Stability horizon.} Long-running stability tests in the current artifact cover up to 24-hour stress executions.
    \item \textbf{Balance regimes.} The main scaling results emphasize near-balanced product dimensions (i.e., $\kappa$ close to $1$). Strongly unbalanced
    regimes are not characterized exhaustively in the reported tables and remain an open evaluation axis.
\end{itemize}

\subsection{Future Directions}
Several extensions naturally follow from the current implementation boundary:
\begin{itemize}
    \item \textbf{Dynamic lattices:} support for adding/removing observers or situations at run time while preserving lattice invariants.
    \item \textbf{Uncertainty-aware credibility:} probabilistic or interval-valued credibility scores to represent uncertainty in belief weights.
    \item \textbf{Richer context models:} support for continuous-valued context parameters and hybrid discrete/continuous context representations.
    \item \textbf{Systems optimizations:} accelerated lattice operations (e.g., GPU-enabled kernels), compressed encodings for sparse lattices, and reduced
    overhead in large distributed deployments.
    \item \textbf{Automated structure discovery:} methods for learning or inferring useful lattice structures from data (e.g., via representation learning or
    data-driven concept discovery).
\end{itemize}

\paragraph{Summary.}
Overall, the provided artifact (code, build support, and tests) enables end-to-end reproduction of the reported experiments and supports validation of both correctness (via unit tests) and scalability (via the experiment driver). The empirical scaling results reported in the paper are obtained on near-balanced lattices, while the theoretical worst-case bounds remain unchanged.

\balance

\end{document}